\documentclass{article}

\usepackage{microtype}
\usepackage{graphicx}
\usepackage{amsmath,bm}
\usepackage{algorithm}
\usepackage{algorithmic}
\usepackage{mathtools}
\usepackage{array,multirow,graphicx}
\usepackage{booktabs} % for professional tables
\usepackage{lipsum}
\usepackage{relsize}
\usepackage{apptools}

\usepackage{float}
\usepackage{caption}
\usepackage{subcaption}
\usepackage{enumitem}
\usepackage{bbm}
\usepackage{titlesec}
\usepackage{amsfonts}
\usepackage{hyperref}
\usepackage{mathtools}
\usepackage{tikz}
\usetikzlibrary{matrix}
\usepackage{verbatim}
\usepackage{mathtools}
\usepackage{caption,booktabs}
\usepackage{tabularx,booktabs}
\usepackage{bbm}
\usepackage{wrapfig,lipsum,booktabs}
\usepackage{capt-of}

\usepackage[accepted]{icml2023}

\usepackage{amsmath}
\usepackage{amssymb}
\usepackage{mathtools}
\usepackage{amsthm}

\newcommand{\ffrac}[2]{\ensuremath{\frac{\displaystyle #1}{\displaystyle #2}}}

\newtheorem{theorem}{Theorem}
\newtheorem{lemma}{Lemma}
\newtheorem{definition}{Definition}
\newtheorem{corollary}{Corollary}
\newtheorem{proposition}{Proposition}

\newtheorem{example}{Example}

\AtAppendix{\counterwithin{proposition}{section}}
\AtAppendix{\counterwithin{definition}{section}}
\AtAppendix{\counterwithin{lemma}{section}}
\AtAppendix{\counterwithin{corollary}{section}}
\AtAppendix{\counterwithin{theorem}{section}}
\AtAppendix{\counterwithin{example}{section}}

\newcommand{\STAB}[1]{\begin{tabular}{@{}c@{}}#1\end{tabular}}
\theoremstyle{remark}
\newtheorem*{remark}{Remark}

\DeclarePairedDelimiter\floor{\lfloor}{\rfloor}
\usepackage[textsize=tiny]{todonotes}

\icmltitlerunning{Hyperbolic Diffusion Embedding and Distance for Hierarchical Representation Learning}

\begin{document}

\twocolumn[
\icmltitle{Hyperbolic Diffusion Embedding and Distance for \\ Hierarchical Representation Learning}




\begin{icmlauthorlist}
\icmlauthor{Ya-Wei Eileen Lin}{comp}
\icmlauthor{Ronald R. Coifman}{sch}
\icmlauthor{Gal Mishne}{yyy}
\icmlauthor{Ronen Talmon}{comp}
\end{icmlauthorlist}

\icmlaffiliation{comp}{Viterbi Faculty of Electrical and Computer Engineering, Technion, Haifa, Israel}
\icmlaffiliation{yyy}{Halicio$\check{\text{g}}$lu Data Science Institute, University of California San Diego, La Jolla, CA, USA}
\icmlaffiliation{sch}{Department of Mathematics, Yale University, New Haven, CT, USA}

\icmlcorrespondingauthor{Ya-Wei Eileen Lin}{lin.ya-wei@campus.technion.ac.il}

% TODO
\icmlkeywords{Hyperbolic geometry, Diffusion geometry, Riemannian geometry, Hierarchical graph representation, Hierarchical data representation, Heat kernel}
% Unsupervised hierarchical metric learning, Hierarchical structure and distance recovery, Classification on hierarchical dataset
\vskip 0.3in
]
%TLDR: We present a new hierarchical embedding and distance recovery method using hyperbolic geometry and diffusion geometry. 



\printAffiliationsAndNotice{} 

\begin{abstract}
Finding meaningful representations and distances of hierarchical data is important in many fields. 
This paper presents a new method for hierarchical data embedding and distance. 
Our method relies on combining diffusion geometry, a central approach to manifold learning, and hyperbolic geometry. 
Specifically, using diffusion geometry, we build multi-scale densities on the data, aimed to reveal their hierarchical structure, and then embed them into a product of hyperbolic spaces. 
We show theoretically that our embedding and distance recover the underlying hierarchical structure. 
In addition, we demonstrate the efficacy of the proposed method and its advantages compared to existing methods on graph embedding benchmarks and hierarchical datasets.  
% \textcolor{blue}{
% We introduce a new method, called Hyperbolic Diffusion embedding and Distance (HDD), for hierarchical representation learning.
% Given observation data, we exploit both diffusion geometry and the Riemannian geometry of hyperbolic space together into a product manifold to represent hierarchical information. 
% Specifically, the diffusion operator approximates the Laplace-Beltrami operator from discrete samples on the tree structure at multiple scales. 
% The equipped multi-scale distance HDD analytically recovers the underlying hierarchical structure, providing a theoretically-grounded guarantee on the intrinsic hierarchy. 
% Empirically, our framework facilitates extracting hierarchical embeddings on benchmark graphs and recovering hierarchical distances on single-cell gene expression data in a reasonable run time. 
% Our HDD has superior embedding quality of mean average precision in representation learning tasks than existing methods and outperforms baselines in downstream classification tasks without geometric priors. 
% }
% Finally, we highlight .....
\end{abstract}

% \textcolor{red}{
% \begin{itemize}
%     \item unsupervised hierarchical metric learning given observation
%     \item ``The quality of a hyperbolic representation is judged by the quality of the metric obtained''
% \end{itemize}}

\section{Introduction}\label{sec:intro}

Hierarchical data is prevalent in many fields of applied science and engineering. Therefore, finding meaningful representations and distances of hierarchical data is an important scientific task.
Hyperbolic geometry provides a powerful tool for this purpose; due to their exponential growth, hyperbolic spaces can naturally represent data with tree-like structures \cite{sarkar2011low}. Indeed, an abundance of methods that involve hyperbolic geometry have been developed in recent years. Notable examples include optimization-based methods \cite{chamberlain2017neural, nickel2017poincare, nickel2018learning, chami2020trees} and combinatorial methods \cite{sala2018representation, sonthalia2020tree}, to mention just a few.
Such hyperbolic geometry-based methods have been successfully applied to central scientific tasks involving hierarchical data in a broad range of fields, e.g., natural language processing \cite{tifrea2018poincar}, social networks \cite{verbeek2014metric}, computer vision \cite{khrulkov2020hyperbolic}, information retrieval \cite{tay2018hyperbolic}, bioinformatics \cite{ding2021deep, lin2021hyperbolic}, and reinforcement learning \cite{cetin2022hyperbolic}.

% Capturing hierarchical information is a key scientific task for understanding the geometric structures spanning various fields, from natural language processing \cite{chamberlain2017neural}, word embeddings \cite{tifrea2018poincar}, social networks \cite{verbeek2014metric}, computer vision tasks \cite{khrulkov2020hyperbolic}, and questionnaires \cite{tay2018hyperbolic} to single-cell gene expression data \cite{dumitrascu2021optimal}. 
% Hyperbolic geometry has attained recent interest in learning hierarchies \cite{nickel2017poincare, sala2018representation} since it provides a natural embedding to depict the exponential growth of tree-like data \cite{sarkar2011low}. 

Despite their growing utility, existing methods for hyperbolic representation learning suffer from several notable shortcomings. For example, methods 
based on optimization of some objective function do not necessarily guarantee optimal representation in terms of standard definitive quality metrics of tree-like geometries \cite{sonthalia2020tree} and could suffer from numerical instabilities \cite{mishne2022numerical}. In the context of this work, perhaps the most restrictive disadvantage of many methods for hyperbolic representation learning is that they require the tree graph or tree distance to be known in advance. However, often in practice, we are only given observational data without any prior information, and the underlying hierarchical structures need to be recovered from the ground up.

In this paper, we present a new method for hierarchical data embedding and distance recovery that provably reveals the underlying tree-like structure. 
In contrast to existing methods, our approach can also be applied to observational data without any prior knowledge of the underlying hierarchical structure.  
Our method builds on diffusion geometry \cite{coifman2006diffusion}, which is a mathematical framework that facilitates the analysis of high-dimensional data points by capturing their underlying geometric structures.
The basic idea behind diffusion geometry is to analyze the similarity between the data points through diffusion propagation. 
It is based on the construction of a diffusion operator from observations, which is tightly related to the heat kernel and the Laplacian of the underlying manifold. Diffusion geometry has been mainly used for manifold learning, giving rise to multi-scale low-dimensional representations and informative distances. In the past two decades, it has been shown useful in a large number of applications from a broad range of fields, for example, spectral clustering \cite{nadler2005diffusion}, signal processing \cite{talmon2013diffusion}, multi-view dimensionality reduction \cite{lindenbaum2020multi}, dynamical systems \cite{talmon2013empirical}, anomaly detection \cite{mishne2012multiscale}, datasets alignment \cite{shnitzer2022log}, variational autoencoders \cite{li2020variational}, graph analysis \cite{cheng2019diffusion}, hyperspectral image clustering \cite{murphy2019spectral}, and non-rigid shape recognition \cite{bronstein2010gromov}, to name but a few.

We propose an embedding of high-dimensional observations and the corresponding distance between observations that extract their underlying hierarchical structure. 
%\textcolor{blue}{The obtained distance metric is used to evaluate the quality of a hierarchical representation. }
Our method conceptually consists of four steps.
First, we build a collection of diffusion operators at multiple scales, aiming to reveal the underlying structure of the data.
Second, we embed each data point using the diffusion operator at each scale into the Poincar\'{e} half-space, which is a particular hyperbolic space model. 
Third, for each data point, we consider the collection of its embedded points at all scales as a point in the product manifold of hyperbolic spaces, whose mutual relationships are designed to recover the hierarchical structure. We term this representation in the product manifold \emph{hyperbolic diffusion embedding} (HDE).
Last, we present the \emph{hyperbolic diffusion distance} (HDD), which naturally stems from the $\ell_1$ distance in the product manifold and the Riemannian distance of the hyperbolic space.
While existing methods of hyperbolic representation learning, e.g., \cite{nickel2017poincare, nickel2018learning}, embed each high-dimensional data point to a new point in hyperbolic space, our method embeds each data point to a new point in a product of hyperbolic spaces (i.e., a collection of points in hyperbolic spaces), thereby further using the hyperbolic geometry for recovering the hierarchical structure.

We posit that our approach is fundamentally different from existing work.
Specifically, our approach combines several components, e.g., diffusion and hyperbolic geometries, for the first time, to the best of our knowledge. 
At first glance, the combination might seem arbitrary, yet, we show theoretically that HDD is equivalent to the underlying hidden hierarchical distance and that each component has a critical contribution to this result. In addition, we demonstrate the applicability of HDD to hierarchical graph benchmarks, single-cell RNA-sequencing data, and unsupervised hierarchical metric learning tasks.
% several downstream classification tasks. 
We show that HDD, compared to existing baselines and deep learning-based methods, achieves improved empirical results, both in terms of two standard quality measures of hierarchical representations and in terms of the accuracy of downstream classification. 
Furthermore, we demonstrate that HDD requires shorter run times than most of the competing methods.

Our main contributions are as follows. 
First, we present a new method for hierarchical data embedding and distance recovery. Our method can receive only observational data as input without prior knowledge.
It is purely data-driven, efficient, and theoretically grounded, and it does not rely on deep learning, optimization, or combinatorial considerations. 
Second, we propose to combine, for the first time to the best of our knowledge, hyperbolic and diffusion geometries. 
We exploit this combination and propose to embed data points through their diffusion operators to hyperbolic spaces in order to build a meaningful multi-scale distance metric.  
Multi-scale distance metrics have been explored in the past using functions defined on Haar wavelet bases \cite{gavish2010multiscale}, partition trees \cite{mishne2016hierarchical,mishne2017data}, and dual manifolds \cite{mishne2019co}.
Here, we show that using hyperbolic spaces, our multi-scale metrics are capable of taking into account multiple scales of the observational data, from the finest to the coarsest, which plays a key role in the recovery of the hierarchical structure.
Third, we showcase improved performance compared to leading recent baselines on several benchmarks, demonstrating accurate and efficient hierarchical structure extraction.

\section{Background}\label{sec:background} 

%We describe the diffusion operator construction for geometric learning, review the notion and distance of product manifold and hyperbolic geometry, and discuss the fidelity metrics of embedding quality.

\paragraph{Diffusion Geometry.}
Diffusion geometry  \cite{coifman2006diffusion} is a framework for high-dimensional data analysis. 
It is based on revealing similarities between the data points by constructing multi-scale ``diffusion'' processes. 
Under the manifold assumption \cite{fefferman2016testing}, i.e., assuming that the high-dimensional data lie on a low-dimensional manifold, diffusion geometry facilitates the recovery of the underlying manifold. 
Below, we outline the main steps of the construction of diffusion geometry that we utilize in our methods.

Let $\mathcal{X} = \{\mathbf{x}_i\}_{i=1}^n$ be a set of data points in an ambient space $\mathbb{R}^m$ that lie on a hidden manifold. 
Let $\mathbf{W}$ be a pairwise affinity matrix, given by 
\begin{equation}
    \mathbf{W}(i,i') = \exp\left( -d^2 (i, i')/\epsilon\right),
    \label{eq:gaussian_kernel}
\end{equation}
where $d(\cdot, \cdot)$ represents a suitable distance between the data points $\mathbf{x}_i$ and $\mathbf{x}_{i'}$, and $\epsilon$ is a tunable kernel scale parameter, which in practice is often set as the median of distances multiplied by a constant or adjusted according to the nearest neighbors \cite{zelnik2004self, keller2009audio, ding2020impact}.
The set $\mathcal{X}$ and the affinity matrix $\mathbf{W}$ form an undirected weighted graph, where $\mathcal{X}$ is the node set and $\mathbf{W}$ is the edges' weight matrix.
By normalizing the affinity matrix twice as follows 
\begingroup
\allowdisplaybreaks
\begin{align}
    \widetilde{\mathbf{W}} &= \mathbf{S}^{-1}\mathbf{W}\mathbf{S}^{-1}, &\mathbf{S}(i,i) &=\textstyle  \sum_j \mathbf{W}(i,j), \label{eq:nor_aff_mat}\\
    \mathbf{P} &= \mathbf{D}^{-1}\widetilde{\mathbf{W}}, &\mathbf{D}(i,i) &= \textstyle \sum_j \widetilde{\mathbf{W}}(i,j), \label{eq:diff_mat}
\end{align}
    \endgroup  
the resulting matrix $\mathbf{P}$ is viewed as a transition probability matrix of a Markov chain defined on the graph. 
% row-stochastic 
%where the density normalized operator $ \tilde{\mathbf{W}}$ addresses the non-uniform data sampling effect \cite{coifman2006diffusion}.

% We call $\mathbf{P}$
The matrix $\mathbf{P}$ is termed \textit{diffusion operator} since it can be used to propagate mass between nodes on the graph.
Let $\mathbf{p}^t_i=\mathbf{P}^t e_i \in \mathbb{R}^n$ be the density on the graph after $t$ diffusion propagation steps, where $e_i \in \mathbb{R}^n$ is the indicator vector of the $i$-th node, and $t\in(0,1]$ is the diffusion time.
Note that by definition, $\mathbf{p}^t_i$ is a well-defined discrete distribution on the graph because $\mathbf{p}^t_i(j) \ge 0$  for any $j=1,\ldots,n$, and $\sum_{j=1}^n \mathbf{p}^t_i(j) = 1$. 
Note that considering multiple diffusion times $t$ gives rise to a collection of multi-scale densities for each data point. 
This construction provides a family of multi-scale distances and embeddings, called diffusion distance and diffusion maps, respectively (see Appendix \ref{app:more_background} for more details).

% Let $e_i \in \mathbb{R}^n$ be the indicator vector of the $i$-th node, which is viewed as a density function on the nodes concentrated at the $i$-th node. 
% Let $\mathbf{P}^t e_i \in \mathbb{R}^n$ be the density on the graph after $t$ diffusion propagation steps, where $t\in(0,1]$ is the diffusion time.

% \textcolor{red}{TODO: Here, we end up with a collection of multiscale densities for each data point. (i) I think we should explicitly note that. (ii) We should include at least a comment about diffusion maps and distance and move their commented-out description to the appendix and just make a reference here.}

The diffusion operator recovers the underlying manifold in the following sense  \cite{coifman2006diffusion}. 
In the limit $n\rightarrow \infty$ and $\epsilon \rightarrow 0$, the operator $\mathbf{P}^{t/\epsilon}$ converges to the Neumann heat kernel of the underlying manifold, given by $\mathbf{H}_t = \exp(-t \Delta)$, where $\Delta$ is the Laplace–Beltrami operator on the manifold.  
In other words, the diffusion operator (matrix) is a discrete approximation of the heat kernel on the manifold.

%  it was shown that $(\mathbf{I} - \mathbf{P})/\epsilon$ converges to the Laplace-Beltrami operator on the manifold $\mathcal{M}$ underlying the data. 
% In addition,

%That is, if we use the language of the graph, we can approximate the diffusion operator by the heat kernel and vice versa. 

\paragraph{Poincar\'{e} Half-Space Model of Hyperbolic Space.}
Hyperbolic geometry is a non-Euclidean geometry with constant negative curvature. 
In this work, we consider the $n$-dimensional Poincar\'{e} half-space model of hyperbolic space with curvature $-1$ \cite{beardon2012geometry}. 
% The Poincar\'{e} half-space model is a hyperbolic space with constant negative curvature $-1$.
% We consider the Poincar\'{e} half-space model \cite{beardon2012geometry}, which is a hyperbolic space with constant negative curvature $-1$.
It is defined by $\mathbb{H}^n = \{ \mathbf{x} \in\mathbb{R}^n \big| \mathbf{x}(n)>0 \}$ with the Riemannian metric tensor $ds^2 =  \frac{d\mathbf{x}^2(1) + d\mathbf{x}^2(2) + \ldots +  d\mathbf{x}^2(n)}{\mathbf{x}^2(n)}$.
Given two points $\bm{x}, \bm{y}\in\mathbb{H}^n$, the Riemannian distance is computed by 
\begin{equation}
    d_{\mathbb{H}^n} (\mathbf{x}, \mathbf{y}) = 2 \sinh^{-1} \left(\frac{\left\lVert\mathbf{x} - \mathbf{y}\right\rVert_2}{2 \sqrt{\mathbf{x}(n)\mathbf{y}(n)}}\right), 
    \label{eq:hyperbolic_geodesic}
\end{equation}
where $\left\lVert\cdot\right\rVert_2$ is the Euclidean norm. 
% TODO: for more details, see...

\paragraph{Product Manifolds and Distances.}
Product manifolds \cite{turaga2016riemannian} provide a product space for mixed curvature representation learning \cite{gu2018learning, Skopek2020Mixedcurvature}. 
Consider a set of Riemannian manifolds denoted by $\{(\mathcal{M}_l,g_l)\}_{l=1}^L$. 
The product manifold is defined by the Cartesian product
$
    \mathcal{M} = \mathcal{M}_1 \times \mathcal{M}_2 \times  \ldots \times \mathcal{M}_L
$, whose dimension is the sum of the dimensions of the factor manifolds $\mathcal{M}_l$.
The product manifold $\mathcal{M}$ is equipped with the Riemannian metric tensor $g = \sum_l g_l$ \cite{ficken1939riemannian}. 
Different distances can be considered in $\mathcal{M}$. In this work, we use the $\ell_1$ distance and show that this choice enables the recovery of the underlying hierarchical structure. The $\ell_1$ is defined by % \cite{gu2018learning}
\begin{equation}
    d_\mathcal{M}^{\ell_1}(x,y) = \sum_{l=1}^L d_{\mathcal{M}_i}(x^l,y^l), 
    \label{eq:product_distance}
\end{equation}
where $x=(x^1,\ldots,x^L), y=(y^1,\ldots,y^L) \in \mathcal{M}$ such that $x^{l},y^{l} \in \mathcal{M}_i$, and $d_{\mathcal{M}_l}$ is the geodesic distance on $\mathcal{M}_l$. 

% Product manifolds \cite{turaga2016riemannian}, on the other hand, not only provide a product space for mixed and various curvatures data learning but also improve the hierarchical representation for tree-like data when considering a product manifold of hyperbolic spaces in a different dimension and with different negative curvatures \cite{gu2018learning, Skopek2020Mixedcurvature}. 
% Nonetheless, a fundamental problem of the dimensionality and curvatures of each component remains open. 

%\section{Hyperbolic Diffusion Distance}

%The procedure embeds the data points at multiple scale in the hyperbolic half-planes, allowing us to capture the underlying hierarchical structure.}

%In Section \ref{sec:HDD_formulation}, we formulate the problems \textcolor{blue}{in graph embedding and hierarchical distance recovery}. \textcolor{blue}{In Section \ref{sec:proposed_method}, we present the proposed method applicable to both problem settings. } In Section \ref{sec:theoretical_ana}, \textcolor{blue}{we present the theoretical guarantee of HDD for approximating }the underlying hierarchical distance and, thus, facilitates the recovery of the hidden hierarchical structure.  

% Algorithm \ref{alg:HDD} summarizes our method. 
% The details and derivations leading to the HDD are presented below. 

\section{Problem Formulation}\label{sec:HDD_formulation}

\begin{figure*}[t]
	\centering
        \includegraphics[width=0.98\textwidth]{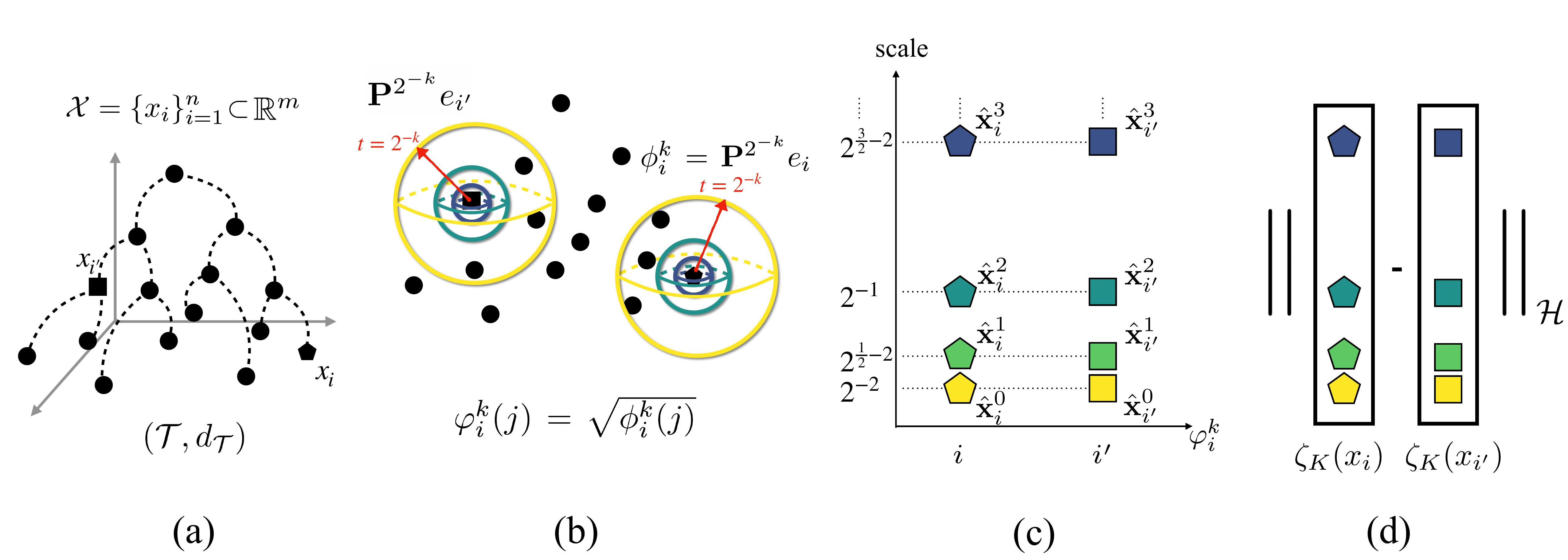}
        \caption{An illustration of HDE and HDD.
        (a) Given a dataset $\mathcal{X}=\{x_i\}_{i=1}^n\subset\mathbb{R}^m$ with an underlying tree-like structure.
        (b) Build a diffusion operator $\mathbf{P}$ by connecting neighboring points. Consider multiple scales of the operator on a dyadic grid $\{\mathbf{P}^{ 2^{-k}}\}_{k=0}^K$ for $K\in\mathbb{Z}_0^+$. For each point $x_i$, define a collection of propagated densities $\phi_i^k=\mathbf{P}^{2^{-k}}e_i$. 
        Here, two densities that are  propagated from two points $x_i$ and $x_{i'}$ are shown at various diffusion times on a dyadic grid.  
        The different scales are represented by using different colors. 
        (c) HDE is given by $\zeta_K:\mathcal{X}\rightarrow\mathcal{H}$, which maps a point $x_i$ into a product of hyperbolic spaces $\mathcal{H} = \mathbb{H}^{n+1} \times \ldots \times \mathbb{H}^{n+1}$ using $\{\varphi_i^k(j) = \sqrt{\phi_i^k(j)}\}_{k=0}^K$. 
        The resulting multi-scale HDE of the two points $x_i$ and $x_{i'}$ are presented, where the colors correspond to the colors in (b). Here, the square root of the propagated densities is shown on the x-axis, and the scales $2^{k\alpha-2}$ of the HDE are shown on the y-axis.
        (d) HDD is defined by the $\ell_1$ distance on the product of hyperbolic spaces} $\mathcal{H}$.
        \label{fig:diagram}
\end{figure*}

   The problem of learning hierarchical representations can be formulated in both graph embedding context, e.g., as in \cite{nickel2017poincare, sala2018representation, sonthalia2020tree}, and in hierarchical distance recovery context, e.g., as in \cite{dasgupta2016cost, klimovskaia2020poincare, chami2020trees, sonthalia2020tree, fang2021kernel}.
%\textcolor{blue}{The effectiveness of hierarchical representation is determined by the accuracy of the metric obtained.}
Their settings and goals are presented below. 

% \subsubsection{Graph Embedding}
In the \textit{graph embedding} context, a tree graph $G=(\mathcal{X}, \mathcal{E}, \mathbf{W})$ is given, where $\mathcal{X}$ is the vertex set with $n$ nodes, $\mathcal{E}$ is the edge set with the weight matrix $\mathbf{W}$. 
In this case, the vertex set $\mathcal{X}$ can be viewed as a discrete subset of a hierarchical metric space $(\mathcal{T},d_{\mathcal{T}})$, where the hierarchical distance  $d_\mathcal{T}$ on the tree nodes $\mathcal{X}$ coincides with the shortest path distance on the graph.
Here, the objective is to find a node embedding into a metric space whose distance approximates the distance $d_{\mathcal{T}}$. 

% \subsubsection{\textcolor{blue}{Hierarchical Distance Recovery}}
In practice, we are often given solely observational data in a high-dimensional ambient space that is assumed to have a latent hierarchical structure that is not explicitly given. Therefore, our primary focus in this work is on the \textit{hierarchical distance recovery} context, where, given a set of $n$ data points $\mathcal{X}=\{\mathbf{x}_i\}_{i=1}^n$ , we view these as a node set of a hidden (tree) graph.
Typically, the points are embedded in some high-dimensional ambient Euclidean space, i.e., $\mathbf{x}_i \in \mathbb{R}^m$, and they have an underlying hierarchical structure. 
This is often formulated by assuming that the points lie in a hidden hierarchical metric space $(\mathcal{T}, d_\mathcal{T})$, where both the space $\mathcal{T}$ and the hierarchical distance $d_\mathcal{T}$ are inaccessible. 
Given the observations $\mathcal{X}$, the goal in hierarchical distance recovery is to find a distance that approximates the hidden hierarchical distance $d_\mathcal{T}$.

% this context is to find a distance that approximates the hidden hierarchical distance $d_\mathcal{T}$ given the observations $\mathcal{X}$. 

\section{Hyperbolic Diffusion Distance}\label{sec:proposed_method}

We begin by considering our primary problem setting of hierarchical distance recovery.
Let $\mathcal{X}=\{\mathbf{x}_i\}_{i=1}^n$ be the data set described above. We follow the construction of the diffusion geometry described in Section \ref{sec:background}. First, we build the matrix $\mathbf{W}$ based on a Gaussian kernel as in Eq.~\eqref{eq:gaussian_kernel}, where $d$ is a suitable distance in the ambient space $\mathbb{R}^m$. Then, the diffusion operator $\mathbf{P}$ is constructed according to Eqs. \eqref{eq:nor_aff_mat} and \eqref{eq:diff_mat}.
Note that this procedure implicitly constructs a graph $G=(\mathcal{X},\mathcal{E},\mathbf{W})$, where $\mathcal{X}$ is viewed as the vertex set and $\mathbf{W}$ is the weight matrix of the edges. Consequently, $\mathbf{P}$ is a transition probability matrix of a random walk on this graph \cite{coifman2006diffusion}.
We present an illustration of the high-dimensional data set $\mathcal{X}$ and the underlying hierarchical structure in Fig.~\ref{fig:diagram}(a). 

This graph viewpoint associates the hierarchical distance recovery and graph embedding contexts. Namely, in the case of graph embedding, the tree graph $G=(\mathcal{X}, \mathcal{E}, \mathbf{W})$ is given, and the diffusion operator $\mathbf{P}$ is defined as an approximation of the heat kernel $\mathbf{H} = \exp(-\mathbf{L})$, as described in Section \ref{sec:background}.
Therefore, from this point on, the two contexts coincide, and the proposed method is applicable to both settings. 

%When the diffusion operator (heat kernel approximation) is built, the embedding and distance are constructed in the same way. 
%In other words, after constructing the diffusion operator $\mathbf{P}$, the following steps are identical. 

We consider a dyadic grid of diffusion times $t = 2^{-k}$ for $k\in\mathbb{Z}_0^+$. Fig.~\ref{fig:diagram}(b) displays the propagated densities at several diffusion times  $t = 2^{-k}$. For convenience, let $\phi_i^k,\varphi_i^k \in \mathbb{R}^n$ denote $\phi_i^k=\mathbf{P}^{2^{-k}}e_{i}$ and $\varphi_i^k(j) = \sqrt{\phi_i^k (j)}$ for $j=1,\ldots,n$.

First, we propose to embed the point $\mathbf{x}_i$ using the diffusion operator $\mathbf{P}^{2^{-k}}$, denoted by the pair $(i,k)$, into the Poincar\'{e} half-space model $\mathbb{H}^{n+1}$ by concatenating $\varphi_i^k$ with a function of the diffusion time $2^{-k}$ as follows: 
\begin{equation}
    (i,k) \mapsto \hat{\mathbf{x}}_i^k = [( \varphi_i^k)^\top, 2^{k\alpha-2}]^\top  \in     \mathbb{H}^{n+1}, 
    \label{eq:hyperbolic_emd}
\end{equation}
where $0<\alpha <1$ is a parameter that scales the diffusion propagation. 
%In practice, $\alpha$ is set to be close to $\frac{1}{2}$.
% We term this embedding the \textit{hyperbolic diffusion embedding} (HDE). 
We remark that the factor $2^{-2}$ added to the diffusion time in Eq.~\eqref{eq:hyperbolic_emd} is an important weight term in the heat kernel approximation, allowing to appropriately capture the local intrinsic association between the hierarchies at each diffusion time scale $t=2^{-k}$.
%
% We begin with a setting of single-scale embedding for a single diffusion time $t\in(0,1]$ and then extend it to all scales at once. 
% When working with the multi-scale framework, we will be concerned with the dyadic scale for the diffusion time such that  $t = 2^{-k}$ for $k\in\mathbb{Z}_0^+$. 
% We first take the point-wise square root of the distribution vector $\phi_i^k=\mathbf{P}^{2^{-k}}e_{i}$,  denoted as $\varphi_i^k(j) = \sqrt{\phi_i^k (j)}$ for $j\in\{1, 2, \ldots, n\}$. 
%
% The hyperbolic diffusion embedding (HDE) at time $t = 2^{-k}$ is a function that maps $\mathcal{X}$ to $\mathbb{H}^{n+1}$ by concatenating $\varphi_i^k$ with a scaling parameter $2^{k\alpha-2}$, given by 
%
 
Next, we extend the embedding by considering multiple diffusion times $\{2^{-k}\}_{k=0}^K$  simultaneously, where $K\in\mathbb{Z}^+_0$ denotes the maximal scale. 
The embedding space is the product manifold $\mathcal{H} = \mathbb{H}^{n+1} \times \mathbb{H}^{n+1} \times \ldots \times \mathbb{H}^{n+1}$ of $(K+1)$ elements (i.e., $\mathcal{H}\subset\mathbb{R}^{(n+1)(K+1)}$), and the multi-scale HDE is a function $\zeta_K:\mathcal{X}\rightarrow\mathcal{H}$ defined by
\begin{equation}
    \zeta_K(x_i) = \left[(\hat{\mathbf{x}}_i^0)^\top, (\hat{\mathbf{x}}_i^1)^\top , (\hat{\mathbf{x}}_i^2)^\top  ,\ldots, ( \hat{\mathbf{x}}_i^K)^\top\right]^\top.
\end{equation}
%where is the product manifold in a dimension $(K+1)(n+1)$ and $\hat{\mathbf{x}}_i^k$ is the single scale embedding in Eq.~\eqref{eq:hyperbolic_emd}.
Fig.~\ref{fig:diagram}(c) illustrates the multi-scale HDE of two points, $\zeta_K(x_i)$ and $\zeta_K(x_{i'})$, where $K=4$.
We used the Poincaré half-space model because of its natural representation of the diffusion time on a dyadic grid (i.e., $t=2^{-k}$ for $k\in\mathbb{Z}_0^+$) as well as its capability to represent multiple diffusion time scales simultaneously.
% Combining diffusion geometry with other hyperbolic models could be a potential extension of the current work. 

Note that considering such a multi-scale embedding significantly departs from the common practice. Existing methods of hierarchical representation learning in the Poincar\'{e} model, e.g., \cite{nickel2017poincare, chami2020trees}, learns the hyperbolic embedding by an optimization method that pushes points toward the boundary of the Poincar\'{e} model or views the data points as tree leaves and attempts to place their embedding close to the boundary of the Poincar\'{e} model.
In contrast, our method embeds each data point as a collection of points in hyperbolic spaces, thereby further exploiting the negative curvature of the space.
%Still, we posit that our embedding maintains information on the given data as existing methods. 
Specifically, by construction, when the diffusion time goes to zero, the diffused densities concentrate at single points, i.e., $\mathbf{P}^t e_i \overset{t\rightarrow 0}{\longrightarrow} e_i$. 
This implies that when the scale $k$ is large (blue in Fig.~\ref{fig:diagram}), the propagated densities $\phi_i^k$ provide a local view of the data, and by  Eq.~\eqref{eq:hyperbolic_emd}, the embedded points at scale $k$ are pushed toward the upper part of the Poinca\'{r}e half space, giving rise to a fine-scale distance. 
Conversely, when $k$ is small (yellow in Fig.~\ref{fig:diagram}), we have a coarse-grained view of the data, and in this case, the proposed embedding gives rise to exponentially scaled distances, which are consistent with the scaling of a tree distance. This argument is made formal in the following statement.
\begin{proposition}\label{prop:exp_growth}
{There is a constant $0<C<1$ such that for any $x_i, x_{i'}\in \mathcal{X}$ and $k_1 \leq k_2$,  $k_1, k_2 \in \mathbb{Z}^+_0$, we have
    \begin{equation}
        C\cdot 2^{-(k_2 - k_1)\alpha}
        \leq
        \frac{d_{\mathbb{H}^{n+1}}(\hat{\mathbf{x}}_i^{k_2},\hat{\mathbf{x}}_{i'}^{k_2} )}{d_{\mathbb{H}^{n+1}}(\hat{\mathbf{x}}_i^{k_1},\hat{\mathbf{x}}_{i'}^{k_1} )} 
        \leq \frac{1}{C} \cdot 2^{-(k_2 - k_1)\alpha}.
    \end{equation} 
    }
    % \begin{equation}
    %     2^{-(k_1 - k_2)\alpha-1}\leq\frac{d_{\mathbb{H}^{n+1}}(\hat{\mathbf{x}}_i^{k_1},\hat{\mathbf{x}}_{i'}^{k_1} )}{d_{\mathbb{H}^{n+1}}(\hat{\mathbf{x}}_i^{k_2},\hat{\mathbf{x}}_{i'}^{k_2} )} \leq 2^{-(k_1 - k_2)\alpha}.
    % \end{equation}
    % \begin{equation}
    %     \frac{d_{\mathbb{H}^{n+1}}(\hat{\mathbf{x}}_i^{k_1},\hat{\mathbf{x}}_{i'}^{k_1} )}{d_{\mathbb{H}^{n+1}}(\hat{\mathbf{x}}_i^{k_2},\hat{\mathbf{x}}_{i'}^{k_2} )} \leq 2^{\lvert k_1 - k_2\rvert\alpha}.
    % \end{equation}
\end{proposition}
% or
% \begin{proposition}
% Let $k_1, k_2 \in \mathbb{Z}^+_0$ and $k_1 \leq k_2$. 
% For any $x_i, x_{i'}\in \mathcal{X}$, $ \frac{d_{\mathbb{H}^{n+1}}(\hat{\mathbf{x}}_i^{k_2},\hat{\mathbf{x}}_{i'}^{k_2} )}{d_{\mathbb{H}^{n+1}}(\hat{\mathbf{x}}_i^{k_1},\hat{\mathbf{x}}_{i'}^{k_1} )}$ is equivalent to $2^{-(k_2 - k_1)\alpha}$. 
% \end{proposition}

The proof of Proposition~\ref{prop:exp_growth} is in Appendix \ref{app:proof_exp_growth}.

% \begin{remark}
% \textcolor{blue}{Existing methods of hierarchical representation learning \cite{nickel2017poincare, nickel2018learning, chami2020trees} push most of the data points toward the boundary of the hyperbolic space that did not exploit the internal structure of the negative-curved space. 
% Our method, on the other hand, explore the dyadic scale in each hyperbolic space and thus the obtained representation in the product space captures the whole granularity and guarantees the recovery of the hierarchical metric. 
% Specifically, observe that when $\hat{\mathbf{x}}_i^k(n+1)\rightarrow 0$, our representation still have the original data points that are closely on $\partial\mathbb{H}^{n+1}$ where the Riemannian distance is now the diffusion distance on the datasets. }    
% \end{remark}

%The embedding is organized from coarse to fine distribution estimations on the graph with the corresponding scaling parameters, where the larger the step $k$, the larger the scaling parameter $2^{k\alpha -2}$. 

\begin{algorithm}[t]
   \caption{Hyperbolic Diffusion Embedding and Distance}
   \label{alg:HDD}
\begin{algorithmic}
    \STATE {\bfseries Input:} Diffusion operator $\mathbf{P}\in\mathbb{R}^{n \times n}$, parameter $\alpha$, and  maximum scale $K$ 
    \STATE {\bfseries Output:} Hyperbolic diffusion distance $d_{\text{HDD}}(i,i')$ for all $i, i'\in[1, n]$ and embedding $\hat{\mathbf{X}}\in\mathbb{R}^{n \times \left( (n+1)(K+1) \right)}$ \vspace{5 mm}
    % \STATE $\mathbf{P} \leftarrow \texttt{Diffusion-Operator}(\mathbf{W})$ %\COMMENT{Eq.~\eqref{eq:aff_mat}-\eqref{eq:diff_mat}}
    \STATE $ \mathbf{U \Lambda V}^\top = \texttt{eig}\left( \mathbf{P}\right)$
    \STATE $k\leftarrow 0$ %, $\mathbf{H}_K\leftarrow \bm{0}\in\mathbb{R}^{n \times n}$
    \WHILE{$k\leq K$}
        \STATE $ \hat{\mathbf{\Lambda}}_k\leftarrow\mathbf{\Lambda}^{2^{-k}} $
        \STATE $\hat{\mathbf{X}}_k \leftarrow [\sqrt{(\mathbf{U} \hat{\mathbf{\Lambda}}_k \mathbf{V}^\top )^\top}, 2^{k\alpha-2} \mathbf{1}_n ]^\top$ 
        % \STATE $\hat{\mathbf{X}}_k \leftarrow \left[\sqrt{\rule{0pt}{2ex}\left(\rule{0pt}{2ex}\mathbf{U} \mathbf{\Lambda}^{2^{-k}} \mathbf{V}^\top \right)^\top}, 2^{\frac{k}{2}-2} \mathbf{1}_n \right]^\top$ 
        % \STATE $\hat{\mathbf{x}}^k_i \leftarrow  [(\sqrt{\mathbf{U} \mathbf{\Lambda}^{2^{-k}} \mathbf{V}^\top e_i })^\top, 2^{\frac{k}{2}-2}]^\top $ 
        \STATE $k\leftarrow k+1$
    \ENDWHILE
    \STATE $\hat{\mathbf{X}}\leftarrow [\hat{\mathbf{X}}_0^\top, \hat{\mathbf{X}}_1^\top, \ldots, \hat{\mathbf{X}}_K^\top]^\top$
    % \FOR{$i \in \{1, 2, \ldots, n\}$}
    %     \STATE $\zeta_K(x_i) \leftarrow   \left[(\hat{\mathbf{x}}_i^0)^\top, (\hat{\mathbf{x}}_i^1)^\top  ,\ldots, ( \hat{\mathbf{x}}_i^K)^\top\right]^\top$
    % \ENDFOR
        \FOR{ $i,i' \in \{1, 2, \ldots, n\}$}
            \STATE  $d_{\text{HDD}}(i,i') \leftarrow \mathlarger{\sum_{k=0}^K } ~  2\sinh^{-1} \left(2^{-k \alpha + 1}\left\lVert \hat{\mathbf{x}}^k_i - \hat{\mathbf{x}}^k_{i'} \right\rVert_2\right)$
        \ENDFOR
        % \STATE $\mathbf{H}_K \leftarrow \mathbf{H}_K + 2 \sinh^{-1} \frac{\texttt{pairwise-distance}(\hat{\mathbf{X}}_k)}{2\cdot 2^{\frac{k}{2}-2}}$ % \COMMENT{Eq.~\eqref{eq:HDD}}
        % \STATE $d_{\text{HDD}}(i,j) \leftarrow d_{\text{HDD}}(i,j) + 2\sinh^{-1} \frac{\left\lVert \hat{\mathbf{x}}_i - \hat{\mathbf{x}}_j\right\rVert_2 }{2^{k+1}}$ for $i,j \in [1, n]$
    \end{algorithmic}
\end{algorithm}

In terms of dimensionality, each diffusion operator embeds a point into an $n$-dimensional density, which in turn is mapped to $\mathbb{H}^{n+1}$. Then, considering the product space of $(K+1)$ scales overall results in an $(n+1)(K+1)$-dimensional HDE. This potential dimensionality increase also departs from common practice that typically aims at dimension reduction. However, as we show in Section \ref{sec:theoretical_ana}, the gain is the ability to recover the underlying hierarchical structure.

Finally, we propose a new distance called the Hyperbolic Diffusion Distance (HDD), using the $\ell_1$ distance on the product manifold $\mathcal{H}$, which is defined in Eqs.~\eqref{eq:hyperbolic_geodesic} and \eqref{eq:product_distance}.
Formally, the HDD between two points $x_i, x_{i'} \in\mathcal{X}$ is defined by 
\begin{align}
    \begin{split}
   d_{\text{HDD}}(i,i') &\coloneqq d_{\mathcal{H}}^{\ell_1}\left( \zeta_K(x_i),  \zeta_K(x_{i'})\right) \\
    &= \sum_{k=0}^K   ~2\sinh^{-1} \left(2^{-k \alpha + 1}\left\lVert \varphi_i^k - \varphi_{i'}^k \right\rVert_2\right),
    \label{eq:HDD}
    \end{split}
\end{align}
where $\alpha$ is the same parameter as in Eq.~\eqref{eq:hyperbolic_emd}. 
Note that the function $\sinh^{-1}$ in Eq.~\eqref{eq:HDD} arises from the Riemannian distance in the hyperbolic space $\mathbb{H}^{n+1}$ at each scale. In practice, it attenuates large distance values.
Fig.~\ref{fig:diagram}(d) illustrates this multi-scale HDD. 
We summarize the construction of the proposed HDE and HDD in Algorithm~\ref{alg:HDD}.

\section{Theoretical Justification}\label{sec:theoretical_ana}

Here we show that the proposed HDE and HDD are theoretically grounded.
The intuition underlying the constructions of HDE and HDD in Algorithm~\ref{alg:HDD} is as follows.
The diffusion operator $\mathbf{P}^t$ is designed to reveal the local connectivity at diffusion time scale $t$. Considering multiple scales in a dyadic grid associates different diffusion timescales, namely, neighborhoods of different sizes. The multi-scale embedding of the corresponding diffusion operators in hyperbolic space naturally endows a hierarchical relationship between the diffusion timescales, enabling the recovery of the underlying hierarchical structure.
Next, we make this intuitive explanation formal.
\begin{theorem}\label{thm:hdd_approx_tree_metric_alpha}
   % Let $(\mathcal{T}, d_\mathcal{T})$ be an hierarchical metric space. 
   For $0<\alpha < \frac{1}{2}$ and sufficiently large $K$ and $n$, the hyperbolic diffusion distance $d_{\text{HDD}}$ is equivalent to $ d_\mathcal{T}^{2\alpha}$.
\end{theorem}

Theorem \ref{thm:hdd_approx_tree_metric_alpha} implies that the proposed HDD recovers the underlying hierarchical distance even when $d_\mathcal{T}$ is not given or when we do not have access to the explicit tree structure. In addition, Theorem \ref{thm:hdd_approx_tree_metric_alpha} suggests that in practice, the parameter $\alpha$ should be set to be close to $\frac{1}{2}$ so that HDD approximates the hierarchical distance. 
That is, for $\alpha \rightarrow \frac{1}{2}$, the obtained $d_{\text{HDD}}$ is approximately $0$-hyperbolic \cite{gromov1987hyperbolic}. 
% \begin{corollary}
%     For $\alpha \rightarrow \frac{1}{2}$, the obtained HDD $d_{\text{HDD}}$ is approximately $0$-hyperbolic, namely, it has $0$ delta hyperbolicity according to Gromov’s definition \cite{gromov1987hyperbolic}.    
% \end{corollary}

Theorem \ref{thm:hdd_approx_tree_metric_alpha} is stated under the assumption that $(\mathbf{P}^t)_{t\in(0,1]}$ is a point-wise approximation of the heat kernel, namely, $\mathbf{P}^t(i,i') \approx a_t(x_i,x_{i'})$, where $a_t(\cdot,\cdot)$ is a heat kernel \cite{grigoryan2009heat}. Such an approximation, mainly in the limit $n \rightarrow \infty$ and $\epsilon \rightarrow 0$, was shown and studied in \cite{coifman2006diffusion,singer2006graph,belkin2008towards}. In addition, three strong regularity conditions are required for the above metric recovery result. Importantly, the heat kernel admits these conditions (see Appendix \ref{app:theoretical_analysis}).
% \textcolor{blue}{The diffusion operator $\mathbf{P}^t$ naturally converges to the \textit{continuous} equivalence  by the law of large number (Monte Carlo integration) \cite{coifman2006diffusion}, and, therefore, the family of operators $(\mathbf{P}^t)_{t\in(0,1]}$, which approximate the heat kernel and are used in this work, satisfy these conditions (see Appendix \ref{app:theoretical_analysis}). 
% }

The first condition is an \textit{upper bound} on the operator elements. There is a non-negative and monotonic decreasing function $f_1: \mathbb{R}_+ \rightarrow \mathbb{R}$ and a number $\beta > 0$ such that for any $\gamma <\beta$, we have $\int_{\mathbb{R}_+} \tau^{3n + \gamma }f_1(\tau)d\tau/\tau < \infty$.
The square root of the operator elements for all $t\in(0, 1]$ is then upper-bounded by 
\begin{equation}
    \sqrt{\mathbf{P}^t(i,i')} \leq \frac{1}{t^{\frac{3n}{2\beta}}}f_1\left(\frac{d_\mathcal{T}(x_i, x_{i'})}{t^{\frac{1}{\beta}}}\right).
\end{equation}

The second condition is a \textit{lower bound} on the operator elements. There is a monotonic decreasing function $g_1:\mathbb{R}_+ \rightarrow \mathbb{R}$ and $R>0$ such that for all $t\in (0,1]$ and all $d_\mathcal{T}(x_i, x_{i'}) < R$, the square root of the operator elements is lower-bounded by 
\begin{equation}
    \sqrt{\mathbf{P}^t(i,i')} \geq \frac{1}{t^{\frac{n}{2\beta}}}g_1\left(\frac{d_\mathcal{T}(x_i, x_{i'})}{t^{\frac{1}{\beta}}}\right).    
\end{equation}

The third condition is a \textit{H\"{o}lder continuity condition}. There is a constant $\Theta >0$ sufficiently small such that for all $t\in (0,1]$, all $x_i, x_{i'} \in \mathcal{X}$ with $d_\mathcal{T}(x_i, x_{i'}) \leq t^{\frac{1}{\beta}}$ and all $x_j\in \mathcal{X}$, the element value of the Hellinger measure  \cite{hellinger1909neue} is upper-bounded by 
\begin{align}
    \begin{split}
        \lvert&\sqrt{\mathbf{P}^t(i,j)} -\sqrt{\mathbf{P}^t(i', j)}\rvert^2 \leq \\
        &\left(\frac{d_\mathcal{T}(x_i, x_{i'})}{t^{\frac{1}{\beta}}}\right)^{2\Theta}\frac{1}{t^{\frac{n}{\beta}}}f_1\left(\frac{d_\mathcal{T}(x_i, x_{i'})}{t^{\frac{1}{\beta}}}\right).     
    \end{split}
\end{align}
Here, we presented only the main results. The proof of Theorem \ref{thm:hdd_approx_tree_metric_alpha} and more details appear in Appendix \ref{app:theoretical_analysis}.

We conclude this subsection with a couple of remarks.
% the following remark. 
First,  our results and derivations rely on and extend the work of \cite{leeb2016holder}, who proposed a multi-scale distance based on diffusion geometry and showed that it approximates the geodesic distance on a closed Riemannian manifold with non-negative curvature \cite{goldberg2012efficient, leeb2015topics}, under the conditions of geometric and semi-groups regularities. 
However, their approximation does not apply to hierarchical structures, such as trees or tree-like structures, that are negatively curved manifolds, as we empirically demonstrate in Appendix \ref{app:additional_exp}.
Here, as a remedy, motivated by the geometric insights presented in \cite{leeb2016holder}, we follow the work of \cite{mckean1970upper, grigor1998heat, frank2013heat, zelditch2017eigenfunctions} for heat kernels on negative curvature spaces. 
% We devise the HDD based on a multi-scale metric using the inverse hyperbolic sine function of a scaled Hellinger distance.   
Second, unlike Euclidean spaces where the product of Euclidean spaces is Euclidean, i.e.,  $(\mathbb{R}^{k_1})^{k_2} = \mathbb{R}^{k_1\cdot k_2}$, considering the product of (curved) hyperbolic spaces gives $(\mathbb{H}^{k_1})^{k_2} 
\neq \mathbb{H}^{k_1\cdot k_2}$, for $k_1, k_2 \in\mathbb{Z}^+$. 
Therefore, embedding into the space $\mathbb{H}^{n(K+1)}$ does not generate the same metric as HDD, which admits a canonical Riemannian metric in $\mathcal{H}$. 
The specific construction of HDD is unique and essential to recover the hierarchy. 
%Moreover, omitting the scaling parameters $2^{k\alpha-2}$ and embedding data points in $\mathbb{H}^{(K+1)n}$ or using a particular single scaling parameter $2^{k\alpha-2}$ in $\mathbb{H}^{(K+1)n+1}$ does not approximate multi-scale kernel density estimations on $\mathcal{T}$. 

% \textcolor{red}{TODO: moved here from intro - incorporate into text. Specifically, the diffusion ground distance on the manifold is approximated by a multi-scale $\ell_1$ norm of the probability measure from the heat kernel \cite{leeb2016holder}
% , we particularly study hierarchy estimation via heat kernels \cite{frank2013heat} and the }

\section{Experimental Results}\label{sec:exp_result}
We investigate the proposed HDE and HDD in hierarchical graph embedding and distance recovery contexts.
Specifically, we apply it to (i) several graphs serving as benchmarks for hierarchical graph embedding, (ii) single-cell gene expression data for the recovery of the hidden hierarchical structure, and (iii) unsupervised hierarchical metric learning tasks.
% dissimilarity-based downstream classification tasks. 
We refer to Appendix \ref{app:more_exp} for more details of these experiments and to Appendix \ref{app:additional_exp} for additional experiments including a toy example and ablation study.
The code is available at the link \url{https://github.com/Ya-Wei0/HyperbolicDiffusionDistance}.

\subsection{Hierarchical Graph Embedding}\label{sec:hierarchial_graph_embedding}

We demonstrate our method in the context of hierarchical graph embedding on five benchmark graphs, which were considered in \cite{sala2018representation}. 
These benchmarks include (i) a small and fully-balanced tree consisting of $40$ nodes, (ii) a phylogenetic tree consisting of $344$ nodes \cite{sanderson1994treebase, hofbauer2016preliminary}, (iii) a graph of disease relations consisting of $516$ nodes, (iv) a CS-PhD graph of the relations between advisors and PhD students consisting of $1025$ nodes \cite{de2018exploratory}, and (v) a general relativity and quantum cosmology arXiv collaboration network with $4185$ nodes \cite{leskovec2007graph}. 
These five graphs contain trees, tree-like graphs, and dense graph.

Given each of these graphs, consisting of nodes and edges, our goal here is to find a node embedding  $\xi(\cdot)$ in a metric space, where the metric $d_\xi$ represents the hierarchical distance $d_\mathcal{T}$ defined by the shortest path on the given graph. 
For this purpose, we apply Algorithm \ref{alg:HDD} with $\alpha = \frac{1}{2}$ and $K\in\{0, 1, \ldots, 19\}$. Implementation details are described in Appendix \ref{app:more_exp}.

\begin{figure}[t]
	\centering
        \includegraphics[width=0.47\textwidth]{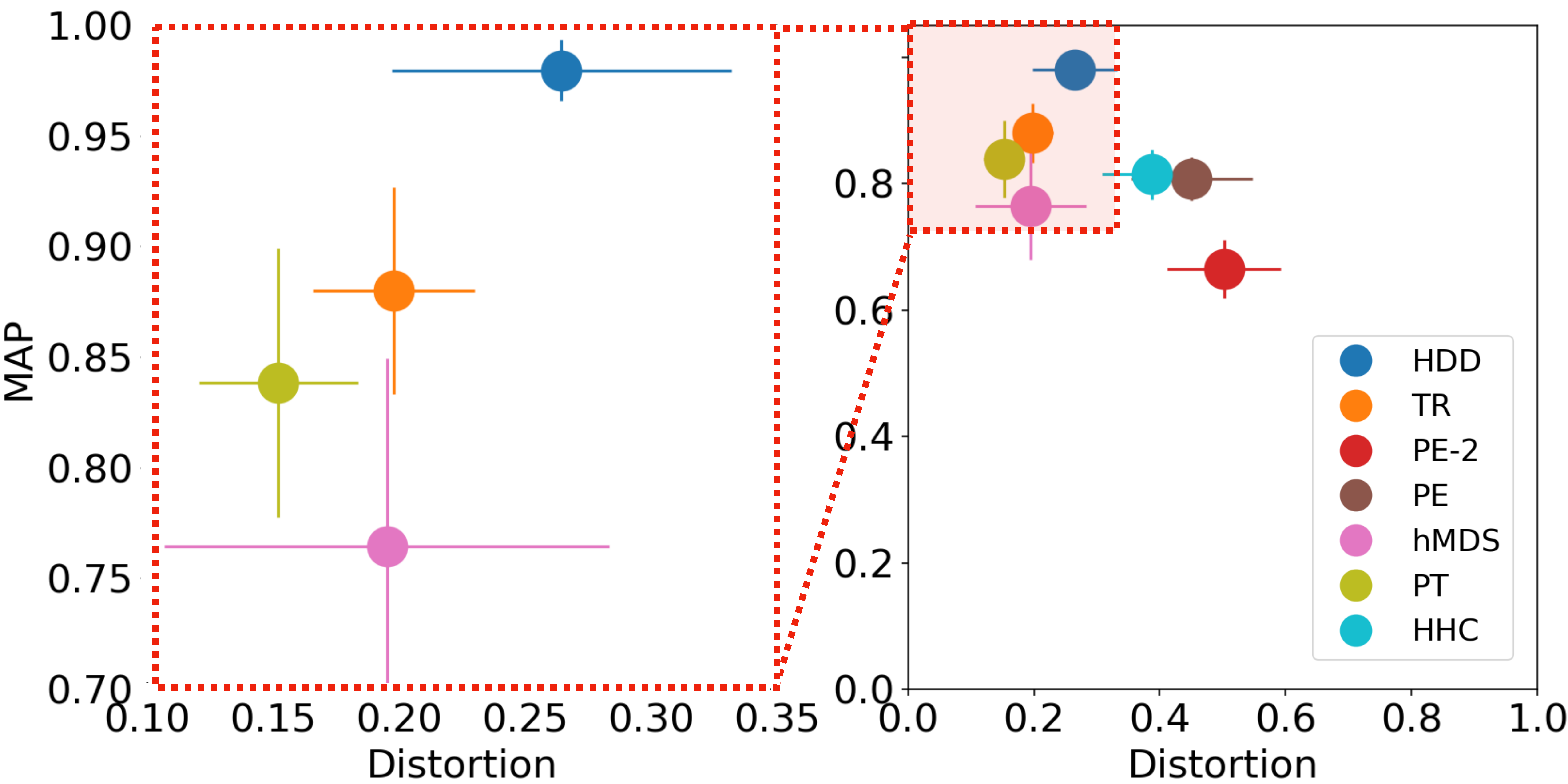}
        \caption{
        Distortion-MAP results of hierarchical graph embedding. Each point represents the mean over the five benchmark graphs, and the whiskers represent the standard deviation. 
        The larger the MAP and the smaller the average distortion, the better the embedding quality.
        }\label{fig:graph_hierarchical_graph}
\end{figure}

We compare our method to five hierarchical embedding methods. The first is the tree representation (TR) obtained by a divide-and-conquer tree construction \cite{sonthalia2020tree}.
The second is the Poincar\'{e} embedding (PE) \cite{nickel2017poincare}, which is a 
neural network for graph embedding that learns the hyperbolic representation in the Poincar\'{e} model.
Two additional methods are taken from \cite{sala2018representation}: a PyTorch (PT) implementation of an SGD-based algorithm optimized over a principal geodesic analysis loss function, and hyperbolic multi-dimensional scaling (hMDS), which takes a pairwise distance matrix as input and returns an embedding in hyperbolic space that best represents the input distances. 
The fifth is the hyperbolic embedding obtained by the hyperbolic hierarchical clustering  (HHC) \cite{chami2020trees} using a continuous relaxation of Dasgupta’s cost \cite{dasgupta2016cost}.
Following the experiment protocol in \cite{sala2018representation}, we test embedding spaces of dimensions 2, 5, 10, 50, 100, and 200 for the PE, hMDS, and PT methods, and the best results are reported. In addition, following the common practice, we also present the results of the PE into the 2-dimensional Poincar\'{e} disk and denote it by PE-2.

To quantitatively evaluate the obtained embeddings, we use two commonly-used fidelity measures. The first is the mean average precision (MAP), which is given by 
\begin{equation}
    \frac{1}{|\mathcal{X}|} \sum_{x_i\in \mathcal{X}}\frac{1}{d(x_i)} \sum_{x_{i_j}\in N_{x_i}} \frac{|N_{x_i} \cap B(\xi(x_i), \xi(x_{i_{j}}))|}{B(\xi(x_i), \xi(x_{i_j}))},
    \label{eq:map}
\end{equation}
where $N_{x_i}= \{x_{i_1}, \ldots, x_{i_{d(x_i)}}\}$ is the set of the neighborhood of $x_i$ in the given graph, $d(x_i)$ is the degree of $x_i$ in the given graph, and $B(x, x_{i_j}) = \{y| d_\xi(\xi(x_i), \xi(y)) \leq d_\xi(\xi(x_i), \xi(x_{i_{j}}))  \} $ is the set of points within the smallest ball that is centered at $\xi(x)$ and contains $x_{i_j}$ in the embedded space.
The second measure is the average distortion given by 
\begin{equation}
     \frac{1}{\begin{psmallmatrix} \lvert \mathcal{X}\rvert \\ 2 \end{psmallmatrix}}  \sum_{x_i\neq x_{i'} \in \mathcal{X}} \frac{\lvert d_\xi(\xi(x_i), \xi(x_{i'})) - d_{\mathcal{T}}(x_i, x_{i'})\rvert }{
    d_{\mathcal{T}}(x_i, x_{i'})}.
    \label{eq:distortion}
\end{equation}
Note that the closer the MAP is to 1, the better the embedding distance $d_\xi$ locally preserves the desired hierarchical distance $d_\mathcal{T}$. 
In addition, the smaller the average distortion is, the larger the (global) similarity between $d_\xi$ and  $d_\mathcal{T}$ is.

Fig.~\ref{fig:graph_hierarchical_graph} presents the MAP and distortion obtained by Algorithm \ref{alg:HDD} and the competing methods. For brevity, we present here the mean and standard deviation over the five benchmarks, and the results for each benchmark separately appear in Appendix \ref{app:additional_exp}.
We see that compared to the baselines, HDD presents a trade-off. It yields the best MAP with a small standard deviation, yet, its obtained distortion is larger than TR, hMDS, and PT.
Note that HDD is strictly better in terms of MAP and distortion than the popular PE in two or more dimensions. 
We report the run time and stability in Appendix \ref{app:additional_exp}, showing that HDD takes a remarkably shorter computational time than PE, PT, and HHC. 
While HDD requires a longer run time than TR, the advantage of HDD over TR in terms of MAP is significant, as depicted in Fig.~\ref{fig:graph_hierarchical_graph}.

%We argue that in the context of graph embedding, both MAP and the average distortion are important. 
% The competing methods TR and hMDS are (implicitly) designed to minimize the average distortion. 
% Yet, the geometric understanding of the product hyperbolic spaces with diffusion operator is used to improve the quality of the local measure. 
% In other words, the multi-scale aggregation in HDD attains the best local and comparable global spectrum of fidelity measures.
% Our proposed method embeds hierarchical information with high fidelity, indicating HDD performed well for hierarchical representation learning from graphs. 

\subsection{Single-Cell Gene Expression Data}\label{sec:scRNA}
We examine HDD in the context of hierarchical distance recovery. In contrast to the graph embedding task, this context fully exploits the use of the proposed multi-scale diffusion geometry that is designed to reveal the hidden hierarchical (tree) structure. For this purpose, we consider single-cell RNA sequencing (scRNA-seq) data \cite{tanay2017scaling}.
It is argued that single-cell development can be well modeled using hierarchical representations \cite{duverle2016celltree}, providing biological insights into cell developmental trajectories and disease progression \cite{van2019single}. 
Revealing the latent hierarchies underlying the cell types is a key task in differentiating genetic treatments and immune responses, which are useful for further biological tasks. 

We test two scRNA-seq datasets taken from \cite{dumitrascu2021optimal}: (i) the mouse cortex and hippocampus dataset (Zeisel) consisting of 3005 single-cells with seven cell types and 4000 gene markers \cite{zeisel2015cell}, and (ii) the cord blood mononuclear cell study (CBMC) comprising 8617 single-cells with 13 cell types and 500 gene markers \cite{stoeckius2017simultaneous}. 

\begin{figure}[t]
	\centering
        \includegraphics[width=0.49\textwidth]{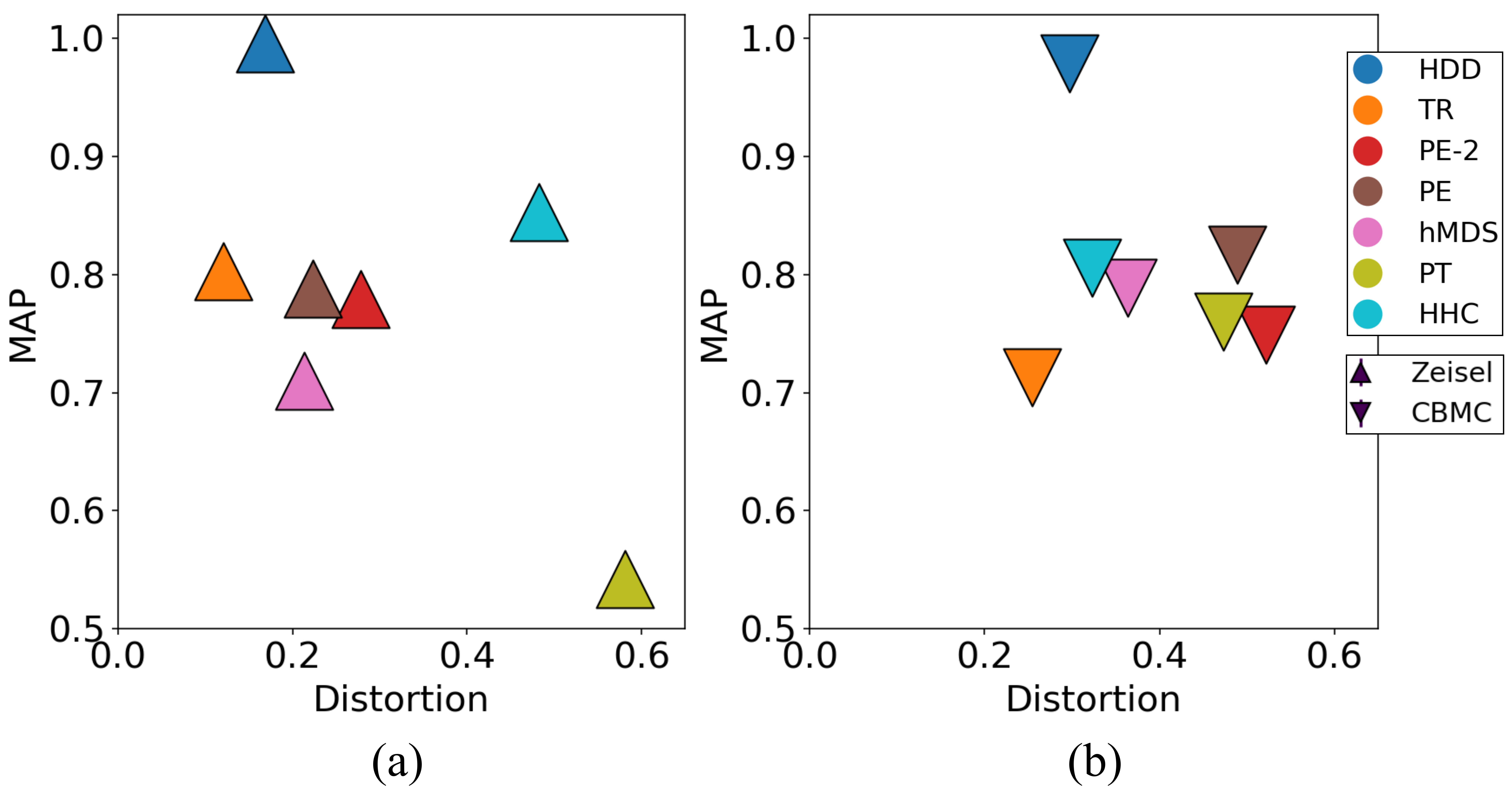}
        \caption{Distortion-MAP results of scRNA-seq datasets: (a) Zeisel and (b) CBMC. The larger the MAP and the smaller the average distortion, the better the hierarchy distance recovery.} 
        \label{fig:graph_gene}
\end{figure}
\begin{figure}[t]
	\centering
        \includegraphics[width=0.49\textwidth]{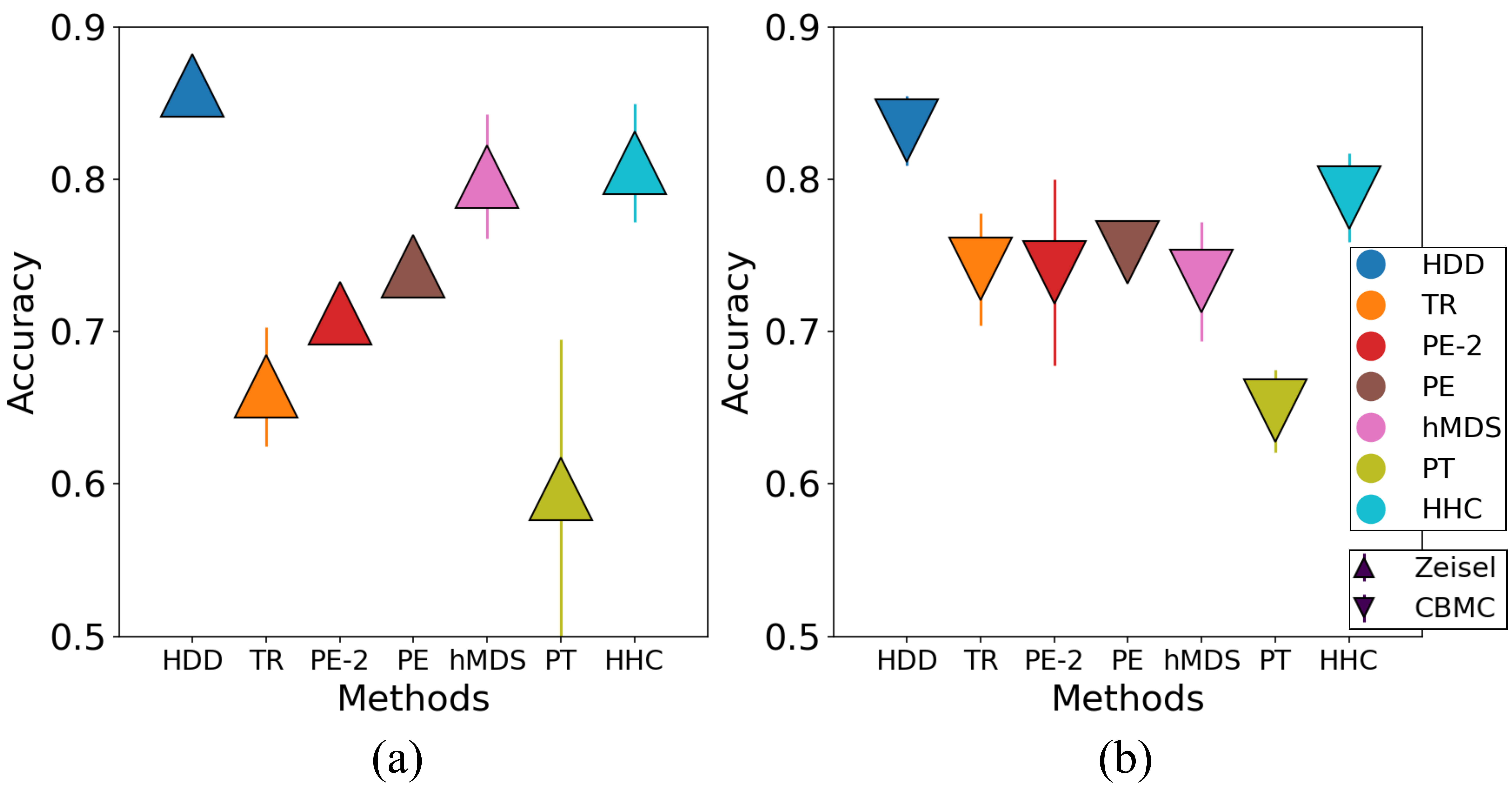}
        \caption{Classification accuracy of scRNA-seq datasets: (a) Zeisel and (b) CBMC. } 
        \label{fig:graph_gene_classification_acc}
\end{figure}
\begin{table*}[t]
% \centering
\caption{Classification accuracy of Zoo, Iris, Glass, and Segmentation datasets (highest accuracy in bold and second highest underlined). } 
% The best accuracy is marked in bold and the  second best is marked by underlined.
\begin{center}
\resizebox{0.96\textwidth}{!}{%
\begin{tabular}{lllccccccccccccr}
\toprule
&  \multicolumn{2}{c}{Dataset \scriptsize{($\#$Points, $\#$Classes)}} & HDD & TR  & PE-2 & PE &  hMDS   & PT & HHC\\
\cmidrule{4-10}
& Zoo &(101, 7)& \textbf{0.898\scriptsize{$\pm$0.012}} & 0.854\scriptsize{$\pm$0.049} &  0.779\scriptsize{$\pm$0.038} &  0.824\scriptsize{$\pm$0.018} &  0.822\scriptsize{$\pm$0.014} &   0.842\scriptsize{$\pm$0.015} & \underline{0.861\scriptsize{$\pm$0.044}} \\
& Iris &(150, 3) & \underline{0.883\scriptsize{$\pm$0.007}} & 0.859\scriptsize{$\pm$0.021} &  0.782\scriptsize{$\pm$0.030} &  0.846\scriptsize{$\pm$0.024} &  0.851\scriptsize{$\pm$0.010} &   \textbf{0.895\scriptsize{$\pm$0.009}} & 0.852\scriptsize{$\pm$0.019} \\
& Glass &(214, 6) &  \textbf{0.654\scriptsize{$\pm$0.011} }& 0.607\scriptsize{$\pm$0.013} & 0.503\scriptsize{$\pm$0.036} &  0.553\scriptsize{$\pm$0.027} &  0.609\scriptsize{$\pm$0.013} &  0.556\scriptsize{$\pm$0.046} &   \underline{0.610\scriptsize{$\pm$0.007}}  \\
& ImaSeg  &(2310, 7) & \textbf{0.701\scriptsize{$\pm$0.024}} & 0.654 \scriptsize{$\pm$0.022} &  0.599\scriptsize{$\pm$0.017} &  \underline{0.679\scriptsize{$\pm$0.038}} &  0.658\scriptsize{$\pm$0.016} &   0.641\scriptsize{$\pm$0.027} & 0.661\scriptsize{$\pm$0.017} \\
\bottomrule
\end{tabular}}
\end{center}
\label{tab:classification_4_datasets}
\end{table*}

The single cells are viewed as samples in a high-dimensional ambient space, where the genes are viewed as features. Given these data, we apply Algorithm \ref{alg:HDD}, obtaining an embedding of the single cells into a hierarchical metric space.
We compare our method to the same baselines as in Section \ref{sec:hierarchial_graph_embedding}. 
A distance based on the cosine similarity computed in the ambient space of the scRNA-seq data expression levels is used as the input distance for Algorithm \ref{alg:HDD} and the competing methods.
We note that the choice of distance in the original space is critical. Here, we employed a distance based on the standard and commonly-used cosine similarity calculated in the ambient space based on prior research \cite{jaskowiak2014selection}.

To evaluate the embedding and distance, we use the same quantitative measures as in Section \ref{sec:hierarchial_graph_embedding}. To compute these measures, we exploit the fact that a tree graph of the cell types is provided with each dataset. Importantly, this information is used only for evaluation but kept hidden from the distance recovery methods. 
%\textcolor{blue}{Specifically, in the computation of MAP and distortion, we evaluate the learned metric with respect to the underlying $\mathcal{T}$. }

Fig.~\ref{fig:graph_gene} presents the obtained MAP and average distortion of Zeisel and CBMC.
We see that HDD obtains the highest MAP values by a large margin and the second-best result in terms of the average distortion, which is very close to the best result. 
The evaluation of the methods' run times (see Appendix \ref{app:additional_exp}) shows that HDD is more efficient than PE, PT, and HHC. %, which we observe take a longer time to converge given a similar number of nodes. 
This suggests that HDD is applicable to large scRNA-seq datasets as well.

% HDD has a high fidelity for extracting the hierarchical relations in the metric learning setting as well. 
% We observe that our deterministic method is linear, which is much more efficient. 
% In addition, we remark that unlike the graph setting for symbolic data as in Section \ref{sec:hierarchial_graph_embedding}, the deep methods PE and PT take a much longer time to converge to the hyperbolic embeddings given a similar number of nodes. 

To further evaluate the results, we make use of the availability of the cell labels in this dataset to also examine HDD through classification. 
Specifically, we apply the nearest centroid classifier with the recovered distance metric of each method as input.
The reported classification accuracy is obtained by averaging over ten different runs; in each run, the dataset is randomly split into 80$\%$ training set and 20$\%$ testing set. 
Fig.~\ref{fig:graph_gene_classification_acc} displays the mean and the standard deviation of the classification accuracy of Zeisel and CBMC.
We see that HDD outperforms the other methods by a large margin for both Zeisel and CBMC datasets. 
This suggests that HDD extracts well the tree-like structure underlying biological data. 
Conversely, we see that TR and hMDS,  which obtain small distortion in the graph embedding context, do not perform well in this downstream task. 

\subsection{Unsupervised Hierarchical Metric Learning}
We further test the proposed HDD in the context of unsupervised hierarchical metric learning.
% downstream classification tasks. 
We consider four datasets from the UCI Machine Learning Repository \cite{Dua2019}: (i) the Zoo dataset consisting of 101 data points of seven types of animals with 17 features, (ii) the Iris dataset comprising 150 samples from three kinds of Iris plants with four features, (iii) the Glass dataset containing 214 instances of six classes with 10 features, and (iv) the image segmentation (ImaSeg) dataset consisting of 2310 instances from seven outdoor images with 19 features.
These datasets were used in \cite{chami2020trees} for evaluating embedding and clustering in hyperbolic space under the working assumption that they have some degree of underlying hierarchical structures.
Here, the evaluation of HDE and HDD is done through downstream classification based on the (dis)similarity of the learned embedding and distance \cite{jordan2015machine}. Such evaluation is affected by the hyperbolicity of the datasets, i.e., it depends on the extent the data adhere to hyperbolic geometry.
Since there is no such ground truth hierarchical information (in contrast to the datasets considered in Section \ref{sec:scRNA}), the $\delta$-hyperbolicity cannot be naively computed. Still, following \cite{chami2020trees}, we posit that comparing different methods for hyperbolic embedding and distance on these datasets gives useful information on the ability of HDD to reveal hierarchical structures given only observational data, compared to the competing methods.

% \textcolor{blue}{
% Last, we measure the HDD quality by performing a downstream (dis)similarity-based classification task. 
% A key fundamental task in modern data analysis is the classification of representation of the data points. 
% Given data, the goal of classification is to measure the quality of a learned structure organizing similar points together and separating the distinct classes. 
% In the context of hierarchical data, the classification object is to evaluate the place of the data points corresponding to the hierarchical structure. 
% Yet, often the time, there is no ground true hierarchy. 
% One common way to perform classification is then to use the obtained pair-wise (dis)similarity of the learned representations \cite{jordan2015machine}. 
% }

The downstream classification is carried out in the same way as in Section \ref{sec:scRNA}.
Algorithm \ref{alg:HDD} and the baseline methods are applied to these datasets with a distance based on cosine similarity as an input without using any label information.
% In these datasets, there is no ground-truth hierarchy as in the scRNA datasets in Section \ref{sec:scRNA}. 
%
Then, a dissimilarity classification based on the learned hierarchical distance using the nearest centroid classifier is applied.
We use cross-validation with ten repetitions, in which the dataset is randomly divided into 80$\%$ training set and 20$\%$ testing set.

% We perform 

The classification accuracy is presented in Table~\ref{tab:classification_4_datasets}, showing the mean and the standard deviation of the results averaged over ten trials.  
We see that HDD outperforms the competing methods in three out of the four datasets and obtains the second-best classification accuracy in the remaining dataset.
These results further demonstrate that HDD gives rise to a useful distance.

\section{Conclusion}\label{sec:conclusion}
We presented a new method for hierarchical data embedding and corresponding distance, termed HDE and HDD, respectively, which can receive as input either a graph or observational data with a hidden hierarchical structure. 
Our method is primarily based on diffusion geometry that enables us to construct multi-scale propagated densities, which are, in turn, embedded in a product of hyperbolic spaces. 
We theoretically show that the $\ell_1$ distance between the embedded points in this space is equivalent to the tree-like distance in the hierarchical space of the input graph or data. In contrast to the common practice in hierarchical data representation using hyperbolic geometry, our method represents each point as a collection of embedded propagated densities rather than a single point in hyperbolic space.
We test HDE and HDD on benchmark graph embedding tasks and on single-cell gene expression data sets, demonstrating significant advantages compared to existing methods in terms of standard quantitative metrics and run time. In addition, we demonstrate that HDD can lead to improved downstream classification accuracy on several benchmarks.
% It can be viewed as the $\delta$-hyperbolic descriptor, facilitating the extraction and representation of hierarchical information.
Because our method is computationally efficient, not based on optimization or deep learning, and differentiable, we posit that it can potentially be incorporated into various loss functions of a broad range of downstream tasks.
For example, combining diffusion and hyperbolic geometry can be extended to deep-based networks \cite{ganea2018hyperbolic, chami2019hyperbolic}, which have been shown useful for hierarchical data learning both theoretically and empirically. 

\section*{Acknowledgements}
We express our gratitude to the anonymous reviewers for their valuable feedback.
The work of YEL and RT was supported by the European Union’s Horizon 2020 research and innovation programme under grant agreement No. 802735-ERC-DIFFOP.
The work of GM was supported by NSF award CCF-2217058.

\bibliography{library}
\bibliographystyle{icml2023}

\newpage
% ~\newpage
\onecolumn
\appendix
\section{Additional Background}\label{app:more_background}
\subsection{Diffusion Geometry}
Broadly, the construction of diffusion geometry starts by defining a probability transition matrix $\mathbf{P}$ that describes how likely it is to transition from one data point to another. This matrix is then used to construct a Markov process, which defines the diffusion distance between data points, conveying a notion of distance between data points based on how easily one can transition or ``diffuse'' from one point to another.
Formally, the diffusion distance with time diffusion $t$ between two points $\mathbf{x}$ and $\mathbf{x}'$ is given by
$\left\lVert \mathbf{P}^t e_i - \mathbf{P}^t e_{i'}\right\rVert$ with an appropriate norm (see \cite{coifman2006diffusion}). 
This diffusion distance is then used to construct a family of multi-scale low-dimensional maps of the data set, termed diffusion maps. 
% The eigenvectors of this Markov process are then used to construct a low-dimensional representation of the data set, known as a diffusion map.
The diffusion maps in $\ell<n$ dimensions with diffusion time $t$ of a point $\mathbf{x}$ is given by
$\Psi_t(\mathbf{x}_i) = [\lambda_1^t \nu_1(i),\ldots,\lambda_{\ell}^t \nu_{\ell}(i)]^\top$, where $\{(\lambda_i, \nu_i)\}_{i=1}^n$ are the eigen-pairs of the transition matrix $\mathbf{P}$.
It is shown that the Euclidean distances between the diffusion maps approximate the diffusion distances \cite{coifman2006diffusion}.
In \cite{coifman2006diffusionwave}, diffusion operators with multiple scales $t$ on a dyadic grid were considered for multi-scale data representation, called diffusion wavelets.

\subsection{Hyperbolic Geometry}
The $n$-dimensional Poincar\'{e} half-space \cite{beardon2012geometry} is a Riemannian manifold with constant negative curvature, defined by $\mathbb{H}^n = \{ \mathbf{x} \in\mathbb{R}^n \big| \mathbf{x}(n)>0 \}$ with the Riemannian metric tensor $ds^2 =  \frac{d\mathbf{x}^2(1) + d\mathbf{x}^2(2) + \ldots +  d\mathbf{x}^2(n)}{a^2\mathbf{x}^2(n)}$, where $a>0$ and $\kappa = -a^2$ represents the Gaussian curvature of the hyperbolic manifold. 
In this work, we study the $n$-dimensional Poincar\'{e} half-space with constant negative curvature $-1$ by setting $a=1$.
% The Riemannian metric tensor of a $n$-dimensional Poincar\'{e} half-space is defined by $ds^2 =  \frac{d\mathbf{x}^2(1) + d\mathbf{x}^2(2) + \ldots +  d\mathbf{x}^2(n)}{a^2\mathbf{x}^2(n)}$.
% , where the constant $a>0$ is analogous to $\frac{1}{R}$ to the sphere. 
% The Gaussian curvature is then defined by $\kappa = -a^2$. 
% By setting $a=1$, we obtain hyperbolic geometry with constant negative curvature $-1$, e.g., the Poincar\'{e} half-space model.
%
% \begin{definition}
%     The $\delta$-hyperbolic metric is a  generalization of the metric on negatively curved manifolds \cite{gromov1987hyperbolic}.
% \end{definition}

The hyperbolic geometry can be characterized by Gromov’s $\delta$-hyperbolicity \cite{gromov1987hyperbolic, ollivier2011visual}. 
% We refer to \cite{ollivier2011visual} for further details. 
\begin{definition}
     A metric space $(X, d)$ is $\delta$-hyperbolic \cite{gromov1987hyperbolic} if there exists $\delta \geq 0$ such that  for all four points $x, y, z, w \in X$ 
    \begin{align}
         d(w, x) + d(y, z) \leq  \max\{d(x, y) + d(w, z),  d(x, z) + d(y, w)\} + 2 \delta. 
        \label{eq:delta_hyp}
    \end{align}
\end{definition}
\begin{proposition}\label{prop:triangle_tree}
    A $0$-hyperbolic metric satisfies the triangle inequality: 
    \begin{equation}
         d(w, x)\leq  d(w, y) +  d(x, y). 
         \label{eq:triangle_tree}
    \end{equation}
\end{proposition}
% \begin{proof}
%     By taking $y=z$ in Eq.~\eqref{eq:delta_hyp} we get an immediate result. 
% \end{proof}
\begin{example}
    The two-dimensional Poincare half-plane $\mathbb{H}^2$ is $(\log 2)$-hyperbolic. 
\end{example}

\subsection{Graph Preliminaries}
\begin{definition}
    Let $G=(V, E, W)$ be an undirected graph. The shortest path metric $d_\mathcal{T}(u,v)$ is the length of the shortest path from $u$ to $v$. 
\end{definition}
\begin{definition}
    A metric $d$ is a 0-hyperbolic metric if there exists a tree $\mathcal{T}$ such that the shortest path metric $d_\mathcal{T}$ on $\mathcal{T}$ is equal to $d$. 
\end{definition}
\section{Proof of Proposition \ref{prop:exp_growth}}\label{app:proof_exp_growth}
\paragraph{Proposition \ref{prop:exp_growth}.}
There is a constant $0<C<1$ such that for any $x_i, x_{i'}\in \mathcal{X}$ and $k_1 \leq k_2$ for $k_1, k_2 \in \mathbb{Z}^+_0$, we have
    \begin{equation}
        C\cdot 2^{-(k_2 - k_1)\alpha}
        \leq
        \frac{d_{\mathbb{H}^{n+1}}(\hat{\mathbf{x}}_i^{k_2},\hat{\mathbf{x}}_{i'}^{k_2} )}{d_{\mathbb{H}^{n+1}}(\hat{\mathbf{x}}_i^{k_1},\hat{\mathbf{x}}_{i'}^{k_1} )} 
        \leq \frac{1}{C} \cdot 2^{-(k_2 - k_1)\alpha}.
    \end{equation} 
    % \begin{equation}
    %     2^{-(k_1 - k_2)\alpha-1}\leq\frac{d_{\mathbb{H}^{n+1}}(\hat{\mathbf{x}}_i^{k_1},\hat{\mathbf{x}}_{i'}^{k_1} )}{d_{\mathbb{H}^{n+1}}(\hat{\mathbf{x}}_i^{k_2},\hat{\mathbf{x}}_{i'}^{k_2} )} \leq 2^{-(k_1 - k_2)\alpha}.
    % \end{equation}
\begin{proof}
    For any $k_1, k_2\in\mathbb{Z}_0^+$ such that $k_1 \leq k_2$ and $x_i, x_{i'}\in\mathcal{X}$, by the bounds of the Hellinger distance, we have $0<c\leq \left\lVert \varphi_i^{k_1} - \varphi_{i'}^{k_1} \right\rVert_2\leq \left\lVert \varphi_i^{k_2} - \varphi_{i'}^{k_2} \right\rVert_2\leq \sqrt{2}$ for some constant $c$.
    We begin with the proof of lower-bound:
    \begingroup
    \allowdisplaybreaks
    \begin{align*}
        \frac{d_{\mathbb{H}^{n+1}}(\hat{\mathbf{x}}_i^{k_2},\hat{\mathbf{x}}_{i'}^{k_2} )}{d_{\mathbb{H}^{n+1}}(\hat{\mathbf{x}}_i^{k_1},\hat{\mathbf{x}}_{i'}^{k_1} )} 
        &= \frac{\sinh^{-1}\left( 2^{-k_2\alpha + 1}\left\lVert \varphi_i^{k_2} - \varphi_{i'}^{k_2} \right\rVert_2\right)}{\sinh^{-1}\left( 2^{-k_1\alpha + 1}\left\lVert \varphi_i^{k_1} - \varphi_{i'}^{k_1} \right\rVert_2\right)}\\
        & \geq  \frac{\sinh^{-1}\left( 2^{-k_2\alpha + 1} \cdot c\right)}{\sinh^{-1}\left( 2^{-k_1\alpha + 1} \cdot \sqrt{2}\right)}\\
        & \overset{(1)}{\geq } \frac{\frac{1}{2}\left( 2^{-k_2\alpha + 1} \cdot c\right)}{\left( 2^{-k_1\alpha + 1} \cdot \sqrt{2}\right)}\\
        &= \frac{c}{2\sqrt{2}}\cdot 2^{-(k_2 - k_1)\alpha},
    \end{align*}
    \endgroup  
    where transition $(1)$ is due to  $\sinh^{-1}(z)<z$ for $z>0$ and $2 \sinh^{-1}(2^{-k\alpha+1} z)> 2^{-k\alpha +1}z$ for $0<z<\sqrt{2}$, $0<\alpha<1$, and $k\in\mathbb{Z}_0^+$. 
    % \begingroup
    % \allowdisplaybreaks
    % \begin{align*}
    %     \frac{d_{\mathbb{H}^{n+1}}(\hat{\mathbf{x}}_i^{k_2},\hat{\mathbf{x}}_{i'}^{k_2} )}{d_{\mathbb{H}^{n+1}}(\hat{\mathbf{x}}_i^{k_1},\hat{\mathbf{x}}_{i'}^{k_1} )} 
    %     &= \frac{\sinh^{-1}\left( 2^{-k_2\alpha + 1}\left\lVert \varphi_i^{k_2} - \varphi_{i'}^{k_2} \right\rVert_2\right)}{\sinh^{-1}\left( 2^{-k_1\alpha + 1}\left\lVert \varphi_i^{k_1} - \varphi_{i'}^{k_1} \right\rVert_2\right)}\\
    %     & \geq  \frac{\sinh^{-1}\left( 2^{-k_2\alpha + 1}\left\lVert \varphi_i^{k_2} - \varphi_{i'}^{k_2} \right\rVert_2\right)}{\sinh^{-1}\left( 2^{-k_1\alpha + 1}\left\lVert \varphi_i^{k_2} - \varphi_{i'}^{k_2} \right\rVert_2\right)}\\
    %     & \geq \frac{ 2^{-k_2\alpha + 1}\cdot c}{ 2^{-k_1\alpha + 1}\cdot c}\\
    %     &= 2^{-(k_2 - k_1)\alpha}.
    % \end{align*}
    % \endgroup    
    Similarly, the upper bound is obtained by 
    \begingroup
    \allowdisplaybreaks
    \begin{align*}
        \frac{d_{\mathbb{H}^{n+1}}(\hat{\mathbf{x}}_i^{k_2},\hat{\mathbf{x}}_{i'}^{k_2} )}{d_{\mathbb{H}^{n+1}}(\hat{\mathbf{x}}_i^{k_1},\hat{\mathbf{x}}_{i'}^{k_1} )} 
        &= \frac{\sinh^{-1}\left( 2^{-k_2\alpha + 1}\left\lVert \varphi_i^{k_2} - \varphi_{i'}^{k_2} \right\rVert_2\right)}{\sinh^{-1}\left( 2^{-k_1\alpha + 1}\left\lVert \varphi_i^{k_1} - \varphi_{i'}^{k_1} \right\rVert_2\right)}\\
        & \leq \frac{\sinh^{-1}\left( 2^{-k_2\alpha + 1}\cdot \sqrt{2}\right)}{\sinh^{-1}\left( 2^{-k_1\alpha + 1} \cdot c\right)}\\
        & \leq  \frac{ 2^{-k_2\alpha + 1}\cdot \sqrt{2}}{ \frac{1}{2} \cdot 2^{-k_1\alpha + 1}\cdot c}\\
        &= \frac{2\sqrt{2}}{c} \cdot 2^{-(k_2 - k_1)\alpha}.
    \end{align*}
    \endgroup    
    Taking $C =\ffrac{c}{2\sqrt{2}}$ gives the results. We remark that the lower bound can be tightly bounded by $ 2^{-(k_2 - k_1)\alpha}$ due to 
    $
    \frac{d_{\mathbb{H}^{n+1}}(\hat{\mathbf{x}}_i^{k_2},\hat{\mathbf{x}}_{i'}^{k_2} )}{d_{\mathbb{H}^{n+1}}(\hat{\mathbf{x}}_i^{k_1},\hat{\mathbf{x}}_{i'}^{k_1} )} = \frac{\sinh^{-1}\left( 2^{-k_2\alpha + 1}\left\lVert \varphi_i^{k_2} - \varphi_{i'}^{k_2} \right\rVert_2\right)}{\sinh^{-1}\left( 2^{-k_1\alpha + 1}\left\lVert \varphi_i^{k_1} - \varphi_{i'}^{k_1} \right\rVert_2\right)} \geq  \frac{\sinh^{-1}\left( 2^{-k_2\alpha + 1}\left\lVert \varphi_i^{k_2} - \varphi_{i'}^{k_2} \right\rVert_2\right)}{\sinh^{-1}\left( 2^{-k_1\alpha + 1}\left\lVert \varphi_i^{k_2} - \varphi_{i'}^{k_2} \right\rVert_2\right)}= 2^{-(k_2 - k_1)\alpha}.
    $
    %   \begingroup
    % \allowdisplaybreaks
    % \begin{align*}
    %     \frac{d_{\mathbb{H}^{n+1}}(\hat{\mathbf{x}}_i^{k_2},\hat{\mathbf{x}}_{i'}^{k_2} )}{d_{\mathbb{H}^{n+1}}(\hat{\mathbf{x}}_i^{k_1},\hat{\mathbf{x}}_{i'}^{k_1} )} 
    %     &= \frac{\sinh^{-1}\left( 2^{-k_2\alpha + 1}\left\lVert \varphi_i^{k_2} - \varphi_{i'}^{k_2} \right\rVert_2\right)}{\sinh^{-1}\left( 2^{-k_1\alpha + 1}\left\lVert \varphi_i^{k_1} - \varphi_{i'}^{k_1} \right\rVert_2\right)}\\
    %     & \leq 2 \cdot \frac{\sinh^{-1}\left( 2^{-k_2\alpha + 1}\cdot \sqrt{2}\right)}{\sinh^{-1}\left( 2^{-k_1\alpha + 1} \cdot \sqrt{2}\right)}\\
    %     & \leq 2 \cdot \frac{ 2^{-k_2\alpha + 1}\cdot \sqrt{2}}{ 2^{-k_1\alpha + 1}\cdot \sqrt{2}}\\
    %     &= 2^{-(k_2 - k_1)\alpha +1}.
    % \end{align*}
    % \endgroup  
\end{proof}

\section{Theoretical Analysis of Hyperbolic Diffusion Distance - Proof of Theorem \ref{thm:hdd_approx_tree_metric_alpha}}\label{app:theoretical_analysis}
The theoretical analysis of HDD is motivated by and derived from the work presented in \cite{leeb2016holder}.
In their work, the authors considered the geometric regularity conditions on the diffusion semi-group of a multi-scale total variation distance between probability measures \cite{goldberg2012efficient, leeb2015topics}.  
More specifically, they presented a diffusion ground distance, a multi-scale distance using $L_1$ distance between probability measures for approximating the geodesic distance on a closed Riemannian manifold with non-negative curvature. 

In our work, we focus on the hierarchical (i.e., tree or tree-like) structures that cannot be approximated by the work in  \cite{leeb2016holder}. 
To this end, we follow the work of \cite{mckean1970upper, grigor1998heat, frank2013heat, zelditch2017eigenfunctions} for spaces with negative curvature and devise the HDD based on a multi-scale metric using inverse hyperbolic sine function of a scaled Hellinger distance \cite{hellinger1909neue}, which forms the $\ell_1$ distance on the product manifold of the hyperbolic spaces. 

First, we define the multi-scale metric HDD in a continuous space and introduce the properties of the diffusion semi-group. 
Next, we establish the geometric regularities in the case of hierarchical datasets that are necessary for the multi-scale metric to approximate the underlying tree metric. 
Last, we will show that the diffusion operators, which approximate the heat kernel, satisfy these conditions, and therefore, the proposed HDD recovers the hierarchical structure underlying the data. 

\subsection{HDD in Continuous Space}\label{app:multi_scale_metric_regularity}

Let $\mathcal{X}$ be a sigma-finite measure space in dimension $n$. 
We consider a measure $\mu$ such that $\mu(B(x, r))\lesssim r^n$, where $x\in\mathcal{X}$ and $r>0$. 
A family of kernels $\{a_t(x, x')\}_{t\in\mathbb{R}_+}$ is considered for $x, x' \in \mathcal{X}$.
Let $f$ be a function defined in $\mathcal{X}$. 
We define the operator $A_t$ by 
\begin{equation}
    A_tf(x) = \int_\mathcal{X} a_t(x, x') f(x') dx'.
\end{equation}
The operators $\{A_t\}_{t\in\mathbb{R}_+}$ have the following properties \cite{coifman2006diffusion, coifman2021some}. 
(i) The family of operators forms a semi-group such that for all $t_1, t_2\in\mathbb{R}_+$, we have $A_{t_1}A_{t_2} = A_{t_1 + t_2}$. 
(ii) The operator respects the conservation property, that is, $\int_\mathcal{X} a_t(x,x')dx' = 1$.
(iii) The operator is integrable such that   $\int_\mathcal{X} |a_t(x,x')|dy \leq C$ for some constant $C>0$.

We are only concerned with dyadic times $t\in(0,1]$ such that  $t = 2^{-k}$ for $k\in\mathbb{Z}_0^+$.
We first define the local geometric measure at a single scale $k$ using the unnormalized Hellinger distance \cite{hellinger1909neue} between the probability distributions, given by 
\begin{equation}
    T_k(x,x') = \left\lVert\sqrt{a_{2^{-k}}(x, \cdot)} - \sqrt{a_{2^{-k}}(x', \cdot)} \right\rVert_2.
    \label{eq:single_hellinger}
\end{equation}
Then we define the multi-scale metric using the inverse hyperbolic sine function of the scaled Hellinger measure with the scaling term $2^{-k\alpha +1}$, given by 
\begin{equation}
   \hat{T}_{\alpha}(x,x') = \sum_{k \geq 0} 2\sinh^{-1}\left( 2^{-k \alpha +1 } \;T_k(x,x')\right) = \sum_{k \geq 0} 2\sinh^{-1} \left(2^{-k \alpha +1 }  \left\lVert\sqrt{a_{2^{-k}}(x, \cdot)} - \sqrt{a_{2^{-k}}(x', \cdot)} \right\rVert_2\right), 
   \label{eq:HDD_continuous}
\end{equation}
where $0<\alpha<1$. 
%is a parameter that measures the depths of descendant and ascendant on $\mathcal{T}$.
Because the scaling parameters decay exponentially, the multi-scale metric can be approximated by the first $K$ terms:
\begin{equation}
   \hat{T}_{\alpha}(x,x')\approx \hat{T}_K(x,x') = \sum_{k=0}^{K} 2\sinh^{-1} \left(2^{-k \alpha + 1} \;T_k(x,x')\right).
   \label{eq:first_K_term}
\end{equation}

\subsection{Regularity Conditions }
We impose geometric regularity on the multi-scale metric $\hat{T}_\alpha$. 
There are constants $C>0$ and $\alpha > 0$ such that the integral of the kernel and the multi-scale metric at scale $k$ is upper-bounded by 
    \begin{align}
         \int_\mathcal{X} a_{2^{-k}}(x,x')\hat{T}_{\alpha}(x,x')  dx' \leq C 2^{-k\alpha}.
         \label{eq:geometric_regularity}
    \end{align}

Let $(\mathcal{X}, d_\mathcal{T})$ be a hierarchical metric space. 
There are three strong regularity conditions imposed on the family of  operators $\{A_{2^{-k}}\}_{k\in\mathbb{R}_+}$ that allow for the proposed multi-metric $\hat{T}_{\alpha}^K(x,x')$ to approximate $d_\mathcal{T}$. 

The first condition is \textit{an upper bound on the kernel}. There is a non-negative and monotonic decreasing function $f_1: \mathbb{R}_+ \rightarrow \mathbb{R}$ and a number $\beta > 0$ such that for any $\gamma <\beta$, we have $\int_{\mathbb{R}_+} \tau^{3n + \gamma }f_1(\tau)d\tau/\tau < \infty$.
The square root of the kernel for all $t\in(0, 1]$ is then upper-bounded by 
\begin{equation}
    \sqrt{a_t(x,x')} \leq \frac{1}{t^{\frac{3n}{2\beta}}}f_1\left(\frac{d_\mathcal{T}(x,x')}{t^{\frac{1}{\beta}}}\right).
    \label{eq:sup_upper_bound_kernel}
\end{equation}
The second condition is \textit{a lower bound of the kernel}. There is a monotonic decreasing function $g_1:\mathbb{R}_+ \rightarrow \mathbb{R}$ and $R>0$ such that for all $t\in (0,1]$ and all $d_\mathcal{T}(x,x') < R$, the square root of the kernel is lower-bounded by 
\begin{equation}
    \sqrt{a_t(x,x')} \geq \frac{1}{t^{\frac{n}{2\beta}}}g_1\left(\frac{d_\mathcal{T}(x,x')}{t^{\frac{1}{\beta}}}\right).    
    \label{eq:sup_lower_bound_kernel}
\end{equation}
The third condition is \textit{H\"{o}lder continuity}. There is a constant $\Theta >0$ sufficiently small such that for all $t\in (0,1]$, all $x,x' \in \mathcal{X}$ with $d_\mathcal{T}(x,x') \leq t^{\frac{1}{\beta}}$ and all $y\in \mathcal{X}$, the element value of the Hellinger measure is upper-bounded by 
\begin{equation}
    \lvert\sqrt{a_t(x,y)} -\sqrt{a_t(x', y)}\rvert^2 \leq 
        \left(\frac{d_\mathcal{T}(x,x')}{t^{\frac{1}{\beta}}}\right)^{2\Theta}\frac{1}{t^{\frac{n}{\beta}}}f_1\left(\frac{d_\mathcal{T}(x,y)}{t^{\frac{1}{\beta}}}\right).  
        \label{eq:sup_holder}
\end{equation}

\subsection{Hierarchical Metric}
We present the lower and upper bounds of the proposed multi-scale metric $\hat{T}_\alpha$ in Eq.~\eqref{eq:HDD_continuous}, making it equivalent to the hierarchical metric $d_\mathcal{T}$. 

\begin{definition}[Snowflake distance \cite{leeb2015topics, leeb2016holder}]
     The snowflake distance is a distance in the form of $d(\cdot, \cdot)^s$, where $d$ is a distance and $0<s<1$. 
\end{definition}
We first present the upper bound of $\hat{T}_\alpha$. 
\begin{proposition}\label{prop:distance_upperbound}
    For any $0<\alpha < \min\{1, \frac{\Theta}{\beta}\}$, the multi-scale $\hat{T}_\alpha$ is upper-bounded by 
    \begin{equation}
        \hat{T}_\alpha(x,x') \lesssim \min \{1, d_\mathcal{T}^{\alpha\beta}(x, x')\}.
    \end{equation}
\end{proposition}
\begin{proof}
    Consider the dyadic levels $K-1$ and $K$ such that the tree distance is bounded from below and above by $2^{-K} \lesssim d_\mathcal{T}^\beta(x,x') \lesssim 2^{-K+1}$. 
    We have 
    \begingroup
    \allowdisplaybreaks
    \begin{align*}
        \hat{T}_\alpha(x,x') & =\sum_{k \geq 0} 2\sinh^{-1} \left(2^{-k \alpha +1 }  \left\lVert\sqrt{a_{2^{-k}}(x, \cdot)} - \sqrt{a_{2^{-k}}(x', \cdot)} \right\rVert_2\right) \\
     & \overset{(1)}{\leq} \sum_{k \geq 0} 2 \cdot 2^{-k \alpha +1 }  \left\lVert\sqrt{a_{2^{-k}}(x, \cdot)} - \sqrt{a_{2^{-k}}(x', \cdot)} \right\rVert_2 \\
     & = \sum_{k \geq 0}  2^{-k \alpha +2 }  \left\lVert\sqrt{a_{2^{-k}}(x, \cdot)} - \sqrt{a_{2^{-k}}(x', \cdot)} \right\rVert_2 \\
     & \overset{(2)}{\lesssim}   d_\mathcal{T}^\Theta(x,x') \sum_{k=0}^K 2^{-k \alpha +2 } 2^{\frac{k \Theta}{\beta}}+ \sum_{k=K+1}^\infty 2^{-k \alpha +2 }\\
    % \lesssim & d_\mathcal{T}^\Theta(x,x') 2^{-K \alpha +2}2^{\frac{K \Theta}{\beta}} + 2^{-K \alpha +2 }\\
    % \lesssim & 2^2 (d_\mathcal{T}^\Theta(x,x') 2^{K \alpha } 2^{\frac{K \Theta}{\beta}}+ 2^{-K \alpha  })\\
    &\lesssim   d_\mathcal{T}^\Theta(x,x') 2^{K \alpha } 2^{\frac{K \Theta}{\beta}}+ 2^{-K \alpha  }\\
    &\overset{(3)}{\lesssim}  d_\mathcal{T}^{\alpha \beta }(x,x'),
    \end{align*}
    \endgroup
    where transition $(1)$ is due to $\sinh^{-1}(z)<z$ for $z>0$, transition $(2)$ is based on the H\"{o}lder continuity condition in Eq.~\eqref{eq:sup_holder} implying that the Hellinger distance is bounded by 
    $\left\lVert\sqrt{a_{2^{-k}}(x,\cdot)} -\sqrt{a_{2^{-k}}(x', \cdot)}\right\rVert_2 \lesssim\left(\frac{d_\mathcal{T}(x, x')}{t^{\frac{1}{\beta}}}\right)^\Theta$, and transition $(3)$ is due to $\alpha < \frac{\Theta}{\beta}$. 
    % Let $x, x'\in\mathcal{X}$, we have 
    % \begingroup
    % \allowdisplaybreaks
    % \begin{align*}
    %     \left\lVert\sqrt{a_t(x,\cdot)} -\sqrt{a_t(x', \cdot)}\right\rVert_2 
    %     \leq \left(\frac{d_\mathcal{T}(x, x')}{t^{\frac{1}{\beta}}}\right)^\Theta\frac{1}{t^{\frac{n}{\beta}}}\int_\mathcal{X} f^2\left(\frac{d_\mathcal{T}(x, u)}{t^{\frac{1}{\beta}}}\right)du \lesssim\left(\frac{d_\mathcal{T}(x, x')}{t^{\frac{1}{\beta}}}\right)^\frac{\Theta}{2}.
    % \end{align*}
    % \endgroup
\end{proof}

Proposition~\ref{prop:distance_upperbound} implies that the upper-bound of the multi-scale metric $ \hat{T}_\alpha$ is a thresholded Snowflake distance of $d_\mathcal{T}$. 
Below, we will demonstrate the lower bound using the following two results. 

% \textcolor{blue}{
% \begin{theorem}\label{thm:geo_regurity}
%     Under conditions upper bounds of the kernel and the H\"{o}lder continuity estimate, Eq.~\eqref{eq:geometric_regularity} holds for $0<\alpha< \frac{\Theta}{\beta}$. 
% \end{theorem}
% \begin{proof}
%     \begin{equation*}
%         \int_\mathcal{X} a_{t}(x,x')\hat{T}_{\alpha}(x,x')  dy\lesssim \int_\mathcal{X}  d_\mathcal{T}^{\alpha \beta }(x,x')  \frac{1}{t^{\frac{3n}{\beta}}}f^2\left(\frac{d_\mathcal{T}(x, y)}{t^{\frac{1}{\beta}}}\right) dy \lesssim  t^\alpha
%     \end{equation*}
% \end{proof}
% }

\begin{lemma}\label{lemma:Hellinger_than_L1}
    Let $p, q$ be two probability distributions on $\mathcal{X}$.
    % Let $\mathbf{p}, \mathbf{q} \in \mathbb{R}^n$ be two distribution vectors such that  $\sum_{i=1}^n \mathbf{p}(i) = \sum_{i=1}^n \mathbf{q}(i) =1$. 
    For a constant $k\in\mathbb{Z}_0^+$ and $0 <\alpha  <1$, we have 
    \begin{equation}
        2 \sinh^{-1} \left(2^{-k\alpha+1}\left\lVert \sqrt{p} - \sqrt{q} \right\rVert_2 \right) \geq 2^{-k\alpha}\lVert p-q\rVert_1, 
    \end{equation}
where $\left\lVert \sqrt{p} - \sqrt{q} \right\rVert_2$ is the unnormalized Hellinger distance between $p$ and $q$.
\end{lemma}

Next, we introduce the lower bound of $\hat{T}_\alpha$. 
\begin{lemma}[Lemma 3 in \cite{leeb2016holder}]\label{lemma:lemma3inholderpaper} 
    Let $R$ be the condition of the lower bound of the kernel in Eq.~\eqref{eq:sup_lower_bound_kernel}.
    There are constants $A>1$ and $\epsilon>0$ such that whenever $x, x'\in\mathcal{X}$ and $t\in(0, 1]$ satisfy $At^{\frac{1}{\beta}} \leq d_\mathcal{T}(x, x')<R$, we have
    \begin{equation}
        \lVert a_{2^{-k}}(x, \cdot)-a_{2^{-k}}(x', \cdot)\rVert_1 \geq \epsilon. 
    \end{equation}
\end{lemma}
\begin{proposition}\label{cor:HDD_lowerbound_tree}
    Let $R$ be as in the condition of the lower bound of the kernel in Eq.~\eqref{eq:sup_lower_bound_kernel}. 
    The multi-scale metric is lower-bounded by 
    \begin{equation}
        \hat{T}_\alpha(x,x') \gtrsim d_\mathcal{T}^{\alpha\beta}(x,x'). 
    \end{equation}
\end{proposition}
\begin{proof}
    Take $A$ and $\epsilon$ as in the conditions in Lemma~\ref{lemma:lemma3inholderpaper}.
    % Let $A$ be a constant such that $A>1$ and $At^{\frac{1}{\beta}} \leq d_\mathcal{T}(x, x')<R$.
    We now take $K\in\mathbb{Z}_+$ such that $\frac{d_\mathcal{T}^\beta(x,x)}{A^\beta}$ is bounded by  $2^{-K} \lesssim \frac{d_\mathcal{T}^\beta(x,x')}{A^\beta} \lesssim 2^{-K+1}$. 
    Then, we have 
    \begingroup
    \allowdisplaybreaks
    \begin{align*}
       \hat{T}_\alpha(x,x') & =\sum_{k \geq 0} 2\sinh^{-1} \left( 2^{-k \alpha +1 }  \left\lVert\sqrt{a_{2^{-k}}(x, \cdot)} - \sqrt{a_{2^{-k}}(x', \cdot)} \right\rVert_2\right)\\
       & \geq  \sum_{k \geq K} 2\sinh^{-1} \left( 2^{-k \alpha +1 }  \left\lVert\sqrt{a_{2^{-k}}(x, \cdot)} - \sqrt{a_{2^{-k}}(x', \cdot)} \right\rVert_2 \right)\\
       &\overset{(1)}{\geq}\sum_{k\geq K} 2^{-k\alpha}\lVert a_{2^{-k}}(x, \cdot)-a_{2^{-k}}(x', \cdot)\rVert_1\\
       & \overset{(2)}{\geq} \epsilon \sum_{k\geq K} 2^{-k\alpha}\\
       & \simeq  2^{-K\alpha }\\
       & \simeq d_\mathcal{T}^{\alpha \beta}(x,x')
    \end{align*}
    \endgroup
    where transition $(1)$ is based on Lemma~\ref{lemma:Hellinger_than_L1} and transition $(2)$ is applied by Lemma~\ref{lemma:lemma3inholderpaper} using the triangle inequality in Eq.~\eqref{eq:triangle_tree}.
\end{proof}
\begin{lemma}[Lemma 4 in \cite{leeb2016holder}]\label{lemma:lemma4inholderpaper} 
    Let $R$ be the condition of the lower bound of the kernel in Eq.~\eqref{eq:sup_lower_bound_kernel}.
    There are constants $C>0$ and $\eta>0$ such that whenever $d_\mathcal{T}(x,x')\geq R$ and $t^{\frac{1}{\beta}} < \eta R$, we have
    \begin{equation}
        \lVert a_{2^{-k}}(x, \cdot)-a_{2^{-k}}(x', \cdot)\rVert_1 \geq C. 
    \end{equation}
\end{lemma}

\begin{proposition}\label{cor:HDD_lowerbound_constantC}
    Let $R$ be the condition of the lower bound of the kernel in Eq.~\eqref{eq:sup_lower_bound_kernel}.
    There is a constant $C>0$ such that  when $d_\mathcal{T}(x, x')>R$, we have 
    \begin{equation}
        \hat{T}_\alpha (x,x') \gtrsim C. 
    \end{equation}
\end{proposition}
\begin{proof}
    Take $C$ and $\eta$ as in the conditions in Lemma~\ref{lemma:lemma4inholderpaper}.
    Let $K =\floor*{\log_2\left(\frac{1}{(\eta R)^\beta}\right)}$. 
    Then $2^{-K} \leq (\eta R)^\beta$, and we have
    \begingroup
    \allowdisplaybreaks
    \begin{align*}
        \hat{T}_\alpha(x,x') & =\sum_{k \geq 0} 2\sinh^{-1} \left( 2^{-k \alpha +1 }  \left\lVert\sqrt{a_{2^{-k}}(x, \cdot)} - \sqrt{a_{2^{-k}}(x', \cdot)} \right\rVert_2 \right)\\
        & \geq \sum_{k \geq K} 2\sinh^{-1} \left( 2^{-k \alpha +1 }  \left\lVert\sqrt{a_{2^{-k}}(x, \cdot)} - \sqrt{a_{2^{-k}}(x', \cdot)} \right\rVert_2 \right)\\
        &\overset{(1)}{\geq } \sum_{k\geq K} 2^{-k\alpha}\lVert a_{2^{-k}}(x, \cdot)-a_{2^{-k}}(x', \cdot)\rVert_1\\
        &\overset{(2)}{\geq }  \sum_{k\geq K} 2^{-k\alpha}C\\
        &\simeq  C (\eta R)^{\alpha\beta},
    \end{align*}
    \endgroup
    where transition $(1)$ is derived by Lemma~\ref{lemma:Hellinger_than_L1} and transition $(2)$ is based on Lemma~\ref{lemma:lemma4inholderpaper}.
\end{proof}
Proposition~\ref{cor:HDD_lowerbound_tree} and Proposition~\ref{cor:HDD_lowerbound_constantC} guarantee that the lower-bound of the multi-scale metric is a thresholded Snowflake distance of $d_\mathcal{T}$. 
We summarize it in the following corollary. 

\begin{corollary}\label{prop:distance_lowerbound}
    Under the conditions of upper and lower bounds of the kernel and the H\"{o}lder continuity in Eq.~\eqref{eq:sup_upper_bound_kernel},  Eq.~\eqref{eq:sup_lower_bound_kernel}, and Eq.~\eqref{eq:sup_holder}, we have
    \begin{equation}
        \hat{T}_\alpha(x,x') \gtrsim \min\{1, d_\mathcal{T}^{\alpha\beta}(x,x')\}. 
    \end{equation}
\end{corollary}

Last, we summarize the equivalence of $\hat{T}_\alpha$ to a thresholded snowflake metric by using  Proposition~\ref{prop:distance_upperbound} and Corollary~\ref{prop:distance_lowerbound}. 
%gives the following proposition.
\begin{proposition}\label{prop:snowflake_tree}
    If the conditions for the upper and lower bounds of the kernel and the H\"{o}lder continuity on $a_{2^{-k}}(x, x')$ hold and if $\mu(B(x, r))\lesssim r^n$, then for $0<\alpha<\min\{1, \frac{\Theta}{\beta}\}$ the distance $\hat{T}_\alpha(x,x')$ is equivalent to the thresholded snowflake distance $\min\{1, d_\mathcal{T}^{\alpha \beta}(x,x')\}$. 
\end{proposition}

\subsection{Heat Kernel on Trees}
In the following, we show that Proposition \ref{prop:snowflake_tree} holds for the \textit{heat kernel on a tree}. For this purpose, we follow \cite{mckean1970upper, grigor1998heat, frank2013heat, zelditch2017eigenfunctions}, showing that the necessary conditions imposed on $\{a_t(x,x')\}_{t \in \mathbb{R}_{+}}$ are satisfied.

In the following lemmas, the operator $a_t$ is considered as the heat kernel on tree. 
\begin{lemma}\label{lemma:heat_upperbound}
    There are constants $A, B>0$ such that  for all $t\in(0,1]$ we have 
    \begin{equation}
        a_t(x,x') \leq \frac{A}{t^{\frac{3n}{2}}}\exp\left(-\frac{B \cdot d_\mathcal{T}^2(x,x')}{t}\right).
    \end{equation}
\end{lemma}
\begin{lemma}\label{lemma:heat_lowerbound}
    There are constants $C, D>0$ such that  for all $t\in(0,1]$ we have 
    % There are constants $C, D>0$ such that  for all $t\in(0,1]$ and $d_\mathcal{T}(x,x')$ is sufficiently small we have 
    \begin{equation}
        a_t(x,x') \geq \frac{C}{t^{\frac{n}{2}}}\exp \left(-\frac{D \cdot d_\mathcal{T}^2(x,x')}{t}\right).
    \end{equation}
\end{lemma}
\begin{lemma}\label{lemma:heat_gradient}
    There are constants $E, F>0$ such that  for $t\in(0,1]$  we have
    \begin{equation}
        \left\lVert\nabla_x a_t(x,x') \right\rVert_2 \leq \frac{E}{\sqrt{t} \cdot t^{\frac{n}{2}}} \exp\left( -\frac{F\cdot d_\mathcal{T}^2(x,x')}{t}\right).
    \end{equation}
\end{lemma}

\begin{proposition}
    \label{prop:holder_heat_tree}
    If $x, x' \in \mathcal{X}$ are sufficiently close, then for a smooth function $h:\mathcal{X}\rightarrow \mathbb
    R_+$, there is a point $y$ lying on the path on $\mathcal{T}$ from $x$ to $y$ such that  
    \begin{equation}
        \lvert h^{\frac{1}{2}}(x)-h^{\frac{1}{2}}(x')\rvert^2\leq \left\lVert \nabla h(y)\right\rVert d_\mathcal{T}^2(x,x').
    \end{equation}
\end{proposition}
\begin{proof}
    Suppose $r\equiv d_\mathcal{T}(x,x')$ is less than the injectivity radius on $\mathcal{X}$. 
    Let $\gamma(t)$ be the unit speed shortest path (with respect to tree) connecting $x$ and $x'$ such that  $\gamma(0) = x$ and $\gamma(r) = x'$.
    Let $\widetilde{h}(t) = h(\gamma(t))$. Note that $\widetilde{h}(0) = h(x)$ and $\widetilde{h}(r) = h(x')$.
    By the mean value theorem, there is some point $0 \leq t'\leq r$ such that  
    \begin{equation*}
        \frac{h(x') - h(x)}{d_\mathcal{T}(x,x')}  = \frac{\widetilde{h}(r) - \widetilde{h}(0)}{r} = \widetilde{h}'(t') = \langle \nabla h (\gamma(t')), \gamma'(t')\rangle.
    \end{equation*}
     Since $\gamma$ has unit speed, by the Cauchy–Schwarz inequality, we have
     \begin{equation*}
         \lvert h(x') - h(x)\rvert = \lvert \langle \nabla h (\gamma(t')), \gamma'(t')\rangle\rvert d_\mathcal{T}(x,x') \leq \left\lVert \nabla h(x')\right\rVert d_\mathcal{T}(x,x').
     \end{equation*}
    In addition, since $h(\cdot)>0$ and the unit speed on $d_\mathcal{T}$, we have
    \begin{equation*}
         \lvert h^{\frac{1}{2}}(x')-h^{\frac{1}{2}}(x)\rvert^2\leq \left\lVert \nabla h(x')\right\rVert d_\mathcal{T}(x,x')\leq \left\lVert \nabla h(x')\right\rVert d_\mathcal{T}^2(x,x').
    \end{equation*}
\end{proof}

\begin{proposition}
    \label{prop:tree_heat_cond2}
    There are positive constants $H, I>0$ such that  for $t\in(0,1]$ and $d_\mathcal{T}(x,x') \lesssim t^{\frac{1}{2}}$, we have
    \begin{equation}
        \lvert\sqrt{a_t(x,u)} -\sqrt{a_t(x', u)}\rvert^2 \leq H\frac{d_\mathcal{T}^2(x,x')}{\sqrt{t}\cdot t^{\frac{n}{2}}} \exp\left( -\frac{I \cdot d_\mathcal{T}^2(u,x)}{t}\right). 
    \end{equation}
\end{proposition}
\begin{proof}
    From Lemma~\ref{lemma:heat_gradient} and Proposition~\ref{prop:holder_heat_tree}, we have
    \begin{equation*}
        \lvert\sqrt{a_t(x,u)} -\sqrt{a_t(x', u)}\rvert^2 \leq d_\mathcal{T}^2(x,x')\frac{E}{\sqrt{t} \cdot t^{\frac{n}{2}}} \exp\left( -\frac{F\cdot d_\mathcal{T}^2(u,y)}{t}\right)
    \end{equation*}
    for some point $y$ lying on the path of $\mathcal{X}$ between $x$ and $x'$.
    Because $d_\mathcal{T}(x,x') \lesssim t^{\frac{1}{2}}$, the inequality  $d_\mathcal{T}(x,y) \lesssim t^{\frac{1}{2}}$ also holds. 
    Then, we have 
    \begin{equation*}
        d_\mathcal{T}^2(u,x) \leq  2\left(d_\mathcal{T}^2(u,y) + d_\mathcal{T}^2(y,x)\right) \lesssim 2\left(d_\mathcal{T}^2(u,y) + t\right)
    \end{equation*}
    and
    \begingroup
    \allowdisplaybreaks
    \begin{align*}
        &\lvert\sqrt{a_t(x,u)} -\sqrt{a_t(x', u)}\rvert^2 \\
        \leq & d_\mathcal{T}^2(x,x') \frac{E}{\sqrt{t} \cdot t^{\frac{n}{2}}} \exp\left( -\frac{F\cdot d_\mathcal{T}^2(u,y)}{t}\right)\\
        \lesssim & d_\mathcal{T}^2(x,x') \frac{E}{\sqrt{t} \cdot t^{\frac{n}{2}}} \exp\left( -\frac{F\cdot (d_\mathcal{T}^2(u,x) - 2t)}{2t}\right)\\
        \leq & d_\mathcal{T}^2(x,x') \frac{E}{\sqrt{t} \cdot t^{\frac{n}{2}}} \exp\left( -\frac{\frac{F}{2}\cdot d_\mathcal{T}^2(u,x)}{t}\right).
    \end{align*}
    \endgroup
\end{proof}

\paragraph{Theorem \ref{thm:hdd_approx_tree_metric_alpha}.}
For $0<\alpha < \frac{1}{2}$ and sufficient $K$, the multi-scale metric $\hat{T}_\alpha^K$  is equivalent to $ d_\mathcal{T}^{2\alpha}$.
\begin{proof}
    By taking $\beta = 2$, Lemma~\ref{lemma:heat_upperbound} and Lemma~\ref{lemma:heat_lowerbound} assure the condition of the upper and lower bounds of the kernel, respectively. 
    Proposition \ref{prop:tree_heat_cond2} ensures the condition of the H\"{o}lder continuity.
    Therefore, by applying Proposition~\ref{prop:snowflake_tree} we can obtain the theorem. 
\end{proof}

\section{Additional Details on the Experimental Study}\label{app:more_exp}
We present the setups and additional details of the experiments in Section \ref{sec:exp_result}.
Our code is included in the supplemental material. 
The experiments are performed on NVIDIA GTX 1080 Ti GPU. 
A fixed random seed \texttt{1234} is used in all the experiments.

% \subsection{Experimental Setup}
\subsection{Baselines}
The implementation of the competing methods is open-source.   
% The baselines are open-sourced. 
The code of tree representation (TR) \cite{sonthalia2020tree} can be found in the open-source implementation\footnote{\label{FootNote:comp_method_TR}\url{https://github.com/rsonthal/TreeRep}}.
We use the PyTorch code in \cite{gu2018learning} for Poincar\'{e} embedding (PE) \cite{nickel2017poincare}, the code in \cite{sala2018representation} for the hyperbolic multi-dimensional scaling (hMDS), and PyTorch (PT) code of an SGD-based algorithm, which are all open-source implementations\footnote{\url{{https://github.com/HazyResearch/hyperbolics}}}.
The code of hyperbolic hierarchical clustering (HHC) \cite{chami2020trees} can be found in the open-source implementation\footnote{\url{https://github.com/HazyResearch/HypHC}}. 
For the graph embedding task, we also consider a two-dimensional hyperbolic embedding built by Sarkar’s combinatorial construction (CC-2) \cite{sarkar2011low} and report the hierarchical graph embedding quality in Table~\ref{tab:hierarchical_graph_MAP_distortion}. 

\subsection{Datasets}
We describe the datasets considered in the experiments in Section \ref{sec:exp_result}. 
They are all publicly available. 
(i) In the hierarchical graph embedding, the hierarchical datasets considered here are structured as graphs with vertices and edges. 
Five benchmark datasets in \cite{sala2018representation}\footnote{\url{https://github.com/HazyResearch/hyperbolics/tree/master/data/edges}} are used, including the small balanced tree, the phylogenetic tree, the disease, the CS-PHD, and the Gr-Qc graphs. 
(ii) In the experiment of scRNA-seq, the datasets are high-dimensional data (samples) measured in an ambient space (gene markers). 
Two open-source datasets in \cite{dumitrascu2021optimal}\footnote{\url{https://github.com/solevillar/scGeneFit-python/tree/}\newline \url{62f88ef0765b3883f592031ca593ec79679a52b4/scGeneFit/data_files}} are considered: Zeisel \cite{zeisel2015cell} and CBMC \cite{stoeckius2017simultaneous}. 
The pre-processing protocol of the scRNA-seq datasets adheres to \cite{dumitrascu2021optimal}. 
(iii) In the downstream classification task, four datasets in the UCI Machine Learning repository \cite{Dua2019}\footnote{\url{https://archive.ics.uci.edu/ml/datasets.php}} are utilized, where the datasets consist of high-dimensional data (instances) collected in an ambient space (attributes).
The datasets we used are the Zoo, the Iris, the Glass, and the Image Segmentation datasets, which are used in \cite{chami2020trees} for hierarchical clustering tasks. 

\subsection{Implementation Details}
For the graph embedding task, the diffusion operator is computed by $\mathbf{P} = \exp(-\mathbf{L})$, where $\mathbf{L}$ is the graph Laplacian matrix. This computation is based on the relationship between the diffusion operator and the heat kernel described in Section \ref{sec:background}.
For high-dimensional data, a distance based on the cosine similarity (\texttt{sklearn.metrics.pairwise\_distances}) computed in the ambient space is used in Eq.~\eqref{eq:gaussian_kernel}, and the diffusion operator is constructed as in Section \ref{sec:background}. 
This distance is also used in the distance-based competing methods, and the corresponding cosine similarity is used in the similarity-based baselines. 
We compute HDD and the embedding according to Algorithm \ref{alg:HDD}, with the parameter $\alpha = \frac{1}{2}$ and the maximal scale $K\in\{0, 1, \ldots, 19\}$.

\begin{remark}
    The computation of HDD in Eq.~\eqref{eq:HDD} involves the diffusion operator construction, calculating the multi-scale distribution vectors, computing the scaled Hellinger distance between data points, and the summation over the inverse hyperbolic sine functions. 
    It could be computationally heavy when working with large-size datasets (i.e., more than ten thousand data points). 
    The construction of the diffusion kernels is typically the most computationally intensive step.
    For large-scale datasets, recent methods in diffusion geometry, such as those presented in \cite{moon2019visualizing, tong2021diffusion, shen2022scalability}, have proposed various techniques (downsampling, interpolative approximations, and landmark diffusion, respectively) to significantly reduce the run time and space complexity of diffusion (e.g., $\Tilde{O}(mn)$ in \cite{tong2021diffusion} instead of  $\Tilde{O}(mn^3)$ and $O(n^{1+2\beta})$ in \cite{shen2022scalability} instead of $O(n^{3})$, where $n$ and $m$ represent the number of samples and features in a data matrix, respectively, and $\beta<1$ is a hyperparameter related to the size of the landmark set).
    These techniques can be integrated into HDD, almost as is, enabling the analysis of datasets larger than ten thousand data points using HDD.
    % Recent studies \cite{moon2019visualizing, tong2021diffusion, shen2022scalability} have introduced methods to significantly reduce the computational load using downsampling, interpolative approximations, and landmark diffusion.
    % These techniques can be incorporated in HDD, allowing for the analysis of datasets of sizes larger than the order of ten thousand.
\end{remark}

\section{Additional Experimental Results} \label{app:additional_exp}
\subsection{Toy Example}
In Fig.~\ref{fig:demo}, we illustrate HDE and HDD on a toy example consisting of a five-level balanced binary tree.
In Fig.~\ref{fig:demo}(a), we plot the given tree graph $G=(\mathcal{T}, \mathcal{E}, \mathbf{W})$, where $\mathcal{T} = \{x_i\}_{i=0}^{31}$ is the vertex set organized from the root to the leaves of the tree, $\mathcal{E}$ is the edge set connecting tree nodes, and $\mathbf{W}$ is the edge connectivity matrix. 
Then, the diffusion operator $\mathbf{P}$ is computed by $\mathbf{P} = \text{exp}(-\mathbf{L})$, where $\mathbf{L}$ is the graph Laplacian of $G$. 
An illustration of the multi-scale propagated densities associated with $\mathbf{P}$ and diffusion times in a dyadic grid is shown in Fig.~\ref{fig:demo}(b). 
We see that the larger the scale $k$, the more local the support of propagated densities, and the smaller the scale $k$, the wider the support of the densities.  
Fig.~\ref{fig:demo}(c) depicts the HDE. Each row represents the multi-scale representation in $\mathcal{H}$, denoted by $\zeta_K(x_i) = \left[(\hat{\mathbf{x}}_i^0)^\top, (\hat{\mathbf{x}}_i^1)^\top ,\ldots, ( \hat{\mathbf{x}}_i^K)^\top\right]^\top$, of each node. Here as well, we see that as the scale increases (from left to right), the representation becomes more local (concentrating at the diagonal).
%The pair-wise square-root of the multi-scale distributions are then mapped to $\mathcal{H}$, a product manifold of the hyperbolic spaces, in which at each factor manifold depicting a single scale $k$ by $ \hat{\mathbf{x}}_i^k = [( \varphi_i^k)^\top, 2^{k\alpha-2}]^\top$ with the parameter $0 < \alpha <1$ for all $x_i\in \mathcal{T}$. Practically, the parameter is set to be close to $\frac{1}{2}$. 
% Finally, the hyperbolic diffusion distance (HDD) is defined by the $\ell_1$ distance in $\mathcal{H}$ as described in Eq.~\eqref{eq:HDD}. 
%For visualization purposes, 
Fig.~\ref{fig:demo}(d) presents the obtained HDD of each node, where the nodes are colored according to their level in the binary tree. For visualization, we depict the two-dimensional multi-dimensional scaling (MDS) \cite{cox2008multidimensional} applied to the nodes using HDD as the input distance. 
We observe that HDD indeed recovers the tree graph. 

\begin{figure}[H]
	\centering
        \includegraphics[width=1\textwidth]{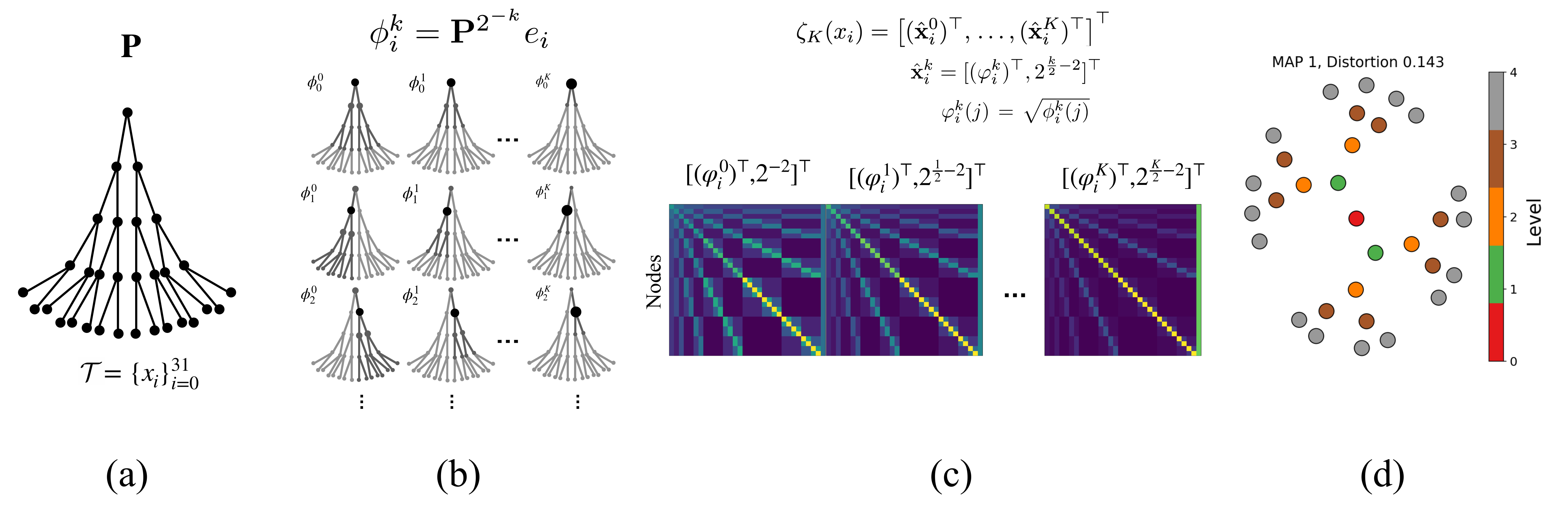}
        \caption{A demonstration of hyperbolic diffusion embedding and distance with an example of a five-level complete binary tree. 
        (a) Given a tree or tree-like graph $G = (\mathcal{T}, \mathcal{E}, \mathbf{W})$, a diffusion operator $\mathbf{P}$ is constructed by the edge connectivity $\mathbf{W}$. 
        (b) The multi-scale propagated densities $\{\phi_i^k\}_{i=0}^{31}$ are computed at each node on $\mathcal{T}$. The rows represent the nodes ordered from the root to the leaves. The columns represent the scale $k$. 
        The size of the nodes and the color of the edges depict the density value at the nodes and the connectivity between them at scale $k$, respectively. 
        (c) The HDE of the nodes is plotted in rows from the small scales (left) to the large scales (right). 
        (d) 2D MDS based on HDD. Each point represents a node. The points are colored by the corresponding levels of the binary tree. 
        }
        \label{fig:demo}
\end{figure}

\subsection{Hierarchical Graph Embedding}

We report the obtained MAP and average distortion for the five hierarchical graph datasets in Table~\ref{tab:hierarchical_graph_MAP_distortion}.
The HDD of the balanced tree, the phylogenetic tree, the disease, the CS-PHD, and the Gr-Qc graphs are respectively obtained with $K=3$, $K=3$, $K=3$, $K=4$, and $K=10$.
We examine the role of maximum scale $K$ in Algorithm \ref{alg:HDD} in the ablation study in Appendix \ref{app:ablation_study}. 
We observe that further increasing the maximum scale does not vary the two fidelity measures, indicating the convergence of our proposed method. 
% We refer to Appendix \ref{app:ablation_study} for the convergence of the maximum scale $K$. 
Table~\ref{tab:hierarchical_graph_MAP_distortion} shows that HDD attains a MAP of 1.0 in the small balanced tree and phylogenetic tree, comparable to the combinatorial representation learned from Sarkar’s construction \cite{sarkar2011low} (CC-2).
For tree-like and dense graphs, HDD outperforms the optimization-based approaches PE, PT, and HHC. 
However, HDD has larger average distortions than TR, hMDS, PT, and CC-2 for most datasets. 
We remark that HDD is strictly better in terms of MAP and average distortion than PE and HHC. 
Arguably, there is a trade-off between the two fidelity measures as noted in \cite{sala2018representation}. 
Our method leans more toward preserving local structure (MAP) at the expense of the global structure (average distortion). 
In Table.~\ref{tab:hierarchical_graph_runtime}, we summarize the execution time for hierarchical representation learning on graphs in this experiment. 
We observe that HDD is the second or the third fastest algorithm for extracting hierarchical information among the five graph datasets. 
While HDD is slower than the divide-and-conquer tree representation (TR), the obtained advantage in MAP values shown in Fig.~\ref{fig:graph_hierarchical_graph} and Table~\ref{tab:hierarchical_graph_MAP_distortion} are significant.
In addition, we report that HDD is much more efficient than the optimization-based methods: PE, PT, and HHC.

\begin{table}[H]
% \centering
\captionof{table}{MAP and average distortion ($D_{\text{avg}}$) of hierarchical graph embedding. } 
\begin{center}
\resizebox{0.68\textwidth}{!}{%
\begin{tabular}{lcccccccccr}
\toprule
&  \multicolumn{1}{c}{\quad}   & \multicolumn{1}{c}{HDD} & \multicolumn{1}{c}{TR} & \multicolumn{1}{c}{PE-2} &   \multicolumn{1}{c}{PE}   &\multicolumn{1}{c}{hMDS}  &   \multicolumn{1}{c}{PT} & \multicolumn{1}{c}{HHC} &   \multicolumn{1}{c}{CC-2} \\
\midrule
\multirow{5}{*}{\STAB{\rotatebox[origin=c]{90}{MAP}}} & Balanced tree & \textbf{1.0} & 0.942 & 0.830 & 0.861  & \textbf{1.0} & 0.964 & 0.901 & \textbf{1.0}\\
& Phylo tree & \textbf{1.0} & 0.931 & 0.696 & 0.724& 0.682 & 0.902 & 0.884  & \textbf{1.0} \\
& Diseases & \textbf{0.970} & 0.873 & 0.611 & 0.912 & 0.931 & 0.943 & 0.831 & 0.808 \\
& CS-PhD & \textbf{0.999} & 0.954 & 0.623 & 0.781& 0.562 & 0.682 & 0.774 & 0.792 \\
& Gr-QC & \textbf{0.930} & 0.701 & 0.564 & 0.763  & 0.649 & 0.702 & 0.685 & 0.684\\
\midrule
\multirow{5}{*}{\STAB{\rotatebox[origin=c]{90}{$D_{\text{avg}}$}}} & Balanced tree & 0.144 & 0.102 & 0.446 & 0.229  & 0.062 & 0.131 & 0.284 & \textbf{0.010}\\
& Phylo tree & 0.520 & 0.304 & 0.841 & 0.641  & 0.087 & 0.207 & 0.696 & \textbf{0.009}\\
& Diseases & 0.206 & 0.187 & 0.426 & 0.694 & 0.123 & \textbf{0.072} & 0.303  & 0.122\\
& CS-PhD & 0.274 & 0.194 & 0.498 & 0.442 & \textbf{0.162} & 0.243 & 0.382  & 0.288\\ 
& Gr-QC & 0.179 & 0.202 & 0.298 & 0.246  & 0.542 & \textbf{0.108} & 0.274& 0.334\\
\bottomrule
\end{tabular}}
\end{center}
\label{tab:hierarchical_graph_MAP_distortion}
\end{table}

\begin{table}[H]
% \centering
\caption{Computation time (in seconds) of hierarchical graph embedding.  } 
\begin{center}
\resizebox{1\textwidth}{!}{%
\begin{tabular}{lllccccccccccccr}
\toprule
&  \multicolumn{2}{c}{Dataset \scriptsize{($\#$Vertices, $\#$Edges)}} & HDD & TR  & PE-2 & PE &  hMDS   & PT & HHC & CC-2\\
\cmidrule{4-11}
&  Balanced tree &(40, 39)& 5.89 $\cdot 10^0$ & \textbf{4.41 $\cdot 10^{-1}$} & 1.22 $\cdot 10^2$ &  8.89 $\cdot 10^2$ &  4.92 $\cdot 10^0$ &   9.48 $\cdot 10^2$ & 7.82 $\cdot 10^1$ & 1.63 $\cdot 10^0$ \\
& Phylo tree &(344, 343) & 4.01 $\cdot 10^1$ & \textbf{9.83 $\cdot 10^{-1}$} & 8.74 $\cdot 10^2$ &  1.24 $\cdot 10^3$ &  6.03 $\cdot 10^1$ &  6.33 $\cdot 10^3$ & 1.62 $\cdot 10^2$ & 2.17 $\cdot 10^0$\\
& Diseases &(516, 1188) & 4.17 $\cdot 10^1$ &\textbf{ 1.02 $\cdot 10^0$} & 1.23 $\cdot 10^3$ &  3.09 $\cdot 10^4$ &  5.21 $\cdot 10^1$ &  1.68 $\cdot 10^4$ &  2.33 $\cdot 10^2$ & 4.42 $\cdot 10^0$\\
& CS-PhD &(1025, 1043) & 6.44 $\cdot 10^1$ &\textbf{ 1.90 $\cdot 10^0$} & 1.78 $\cdot 10^4$ &  5.62 $\cdot 10^4$ &  9.63 $\cdot 10^1$ &  2.50 $\cdot 10^4$ &  5.43 $\cdot 10^2$ &  7.93 $\cdot 10^0$\\
& Gr-QC  &(4158, 13428) & 9.12 $\cdot 10^1$ & \textbf{2.03 $\cdot 10^0$} & 2.89 $\cdot 10^4$ &  3.13 $\cdot 10^5$ &  1.94 $\cdot 10^2$ &  3.43 $\cdot 10^4$ &  1.92 $\cdot 10^3$ &  9.37 $\cdot 10^1$\\
\bottomrule
\end{tabular}}
\end{center}
\label{tab:hierarchical_graph_runtime}
\end{table}

% The HDDs for each dataset are respectively obtained at $K=3$, $K=3$, $K=3$, $K=4$, and $K=10$.

% We appeal that there is a trade-off between MAP and the average distortion, in which HDD is more toward the local spectrum while preserving decent global measure. 
% Specifically, in the aspect of the method that counts on optimization for MAP measure, HDD is leading in the fidelity global-local spectrum.  
\subsection{Ablation Study}\label{app:ablation_study}

We conduct an ablation study to investigate the effectiveness of the different components in our method. 
First, we compare HDD with the $\ell_2$ distance in the product manifold $\mathcal{H}$, given by
\begin{equation}
d_{\mathcal{H}}^{\ell_2}\left( \zeta_K(x_i),  \zeta_K(x_{i'})\right) 
    = \sum_{k=0}^K   \left( 2\sinh^{-1} \left(2^{-k \alpha + 1}\left\lVert \varphi_i^k - \varphi_{i'}^k \right\rVert_2\right)\right)^2,
\end{equation}
where $K\in\mathbb{Z}_0^+$ is the maximum scale defined in the same way as in HDD. 
Note that $d_\mathcal{H}^{\ell_2}$ is equipped with a Riemannian structure \cite{ficken1939riemannian}. 
In addition, we test single scales in the factor manifold $\mathbb{H}^{n+1}$ in the product manifold $\mathcal{H}$, given by 
\begin{equation}
    d_{\mathbb{H}^{n+1}}(\hat{\mathbf{x}}_i^{k},\hat{\mathbf{x}}_{i'}^{k}) =2\sinh^{-1}\left( 2^{-k\alpha + 1}\left\lVert \varphi_i^{k} - \varphi_{i'}^{k} \right\rVert_2\right), 
\end{equation}
where $k\in\mathbb{Z}_0^+$ represents the scale. 

The results, comparing HDD, $d_\mathcal{H}^{\ell_2}$, and the single scale embedding are presented in Fig.~\ref{fig:ablation_map_distortion_graph} for the graph embedding experiment presented in Section \ref{sec:hierarchial_graph_embedding}. 
The five plots depict the distortion-MAP graph for the five datasets.
In each plot, the blue circle, green plus, and red star represent the result of HDD, $d_\mathcal{H}^{\ell_2}$, and the single scale embedding, respectively. 
The color of the points represents the parameter $K$ (resp. $k$) for HDD and $d_\mathcal{H}^{\ell_2}$ (resp. single scale embedding).  
Note that when $k = K = 0$, HDD and the single scale embedding coincide (i.e., $d_\mathcal{H}^{\ell_1}\left( \zeta_0(x_i),  \zeta_0(x_{i'})\right)  =  d_{\mathbb{H}^{n+1}}(\hat{\mathbf{x}}_i^{0},\hat{\mathbf{x}}_{i'}^{0}) $). 
We observe that HDD outperforms the other two alternatives, indicating that, indeed, the use of the $\ell_1$ norm and the multiple scales in Eq.~\eqref{eq:HDD} has a critical contribution to the extraction of the hierarchical structure, as guaranteed in Theorem \ref{thm:hdd_approx_tree_metric_alpha}.
In addition, based on the results of HDD and $d_\mathcal{H}^{\ell_2}$, we find that the larger $K$ is, the better the embedding quality is. 
Conversely, the role of $k$ plays an opposite effect in the single embedding, as conveyed in Proposition~\ref{prop:exp_growth}. 
Last, we see that the results of HDD in terms of MAP and average distortion converge for sufficiently large $K$, providing empirical support to the approximation in Eq.~\eqref{eq:first_K_term}. 
% which verified hierarchical distance can be approximated by the first $K$ terms as described in 

% We conclude the ablation study with the following remark. 

% Unlike Euclidean spaces where the product of Euclidean spaces is still Euclidean, i.e.,  $(\mathbb{R}^{k_1})^{k_2} = \mathbb{R}^{k_1\cdot k_2}$, it does not apply to curved spaces, e.g.,  in hyperbolic spaces $(\mathbb{H}^{k_1})^{k_2} 
% \neq \mathbb{H}^{k_1\cdot k_2}$, for $k_1, k_2 \in\mathbb{Z}^+$. 
% Embedding to the space $\mathbb{H}^{(K+1)n}$  does not generate the same metric as HDD, which admits a canonical Riemannian metric in $\mathcal{H}$. 
% Moreover, omitting the scaling parameters in $\mathbb{H}^{(K+1)n}$ or using a particular single scaling parameter in $\mathbb{H}^{(K+1)n+1}$ does not approximate multi-scale kernel density estimations on $\mathcal{T}$. 
% The specific construction of HDD is unique and essential to recover the hierarchy. 

\begin{figure}[H]
	\centering
        \includegraphics[width=1\textwidth]{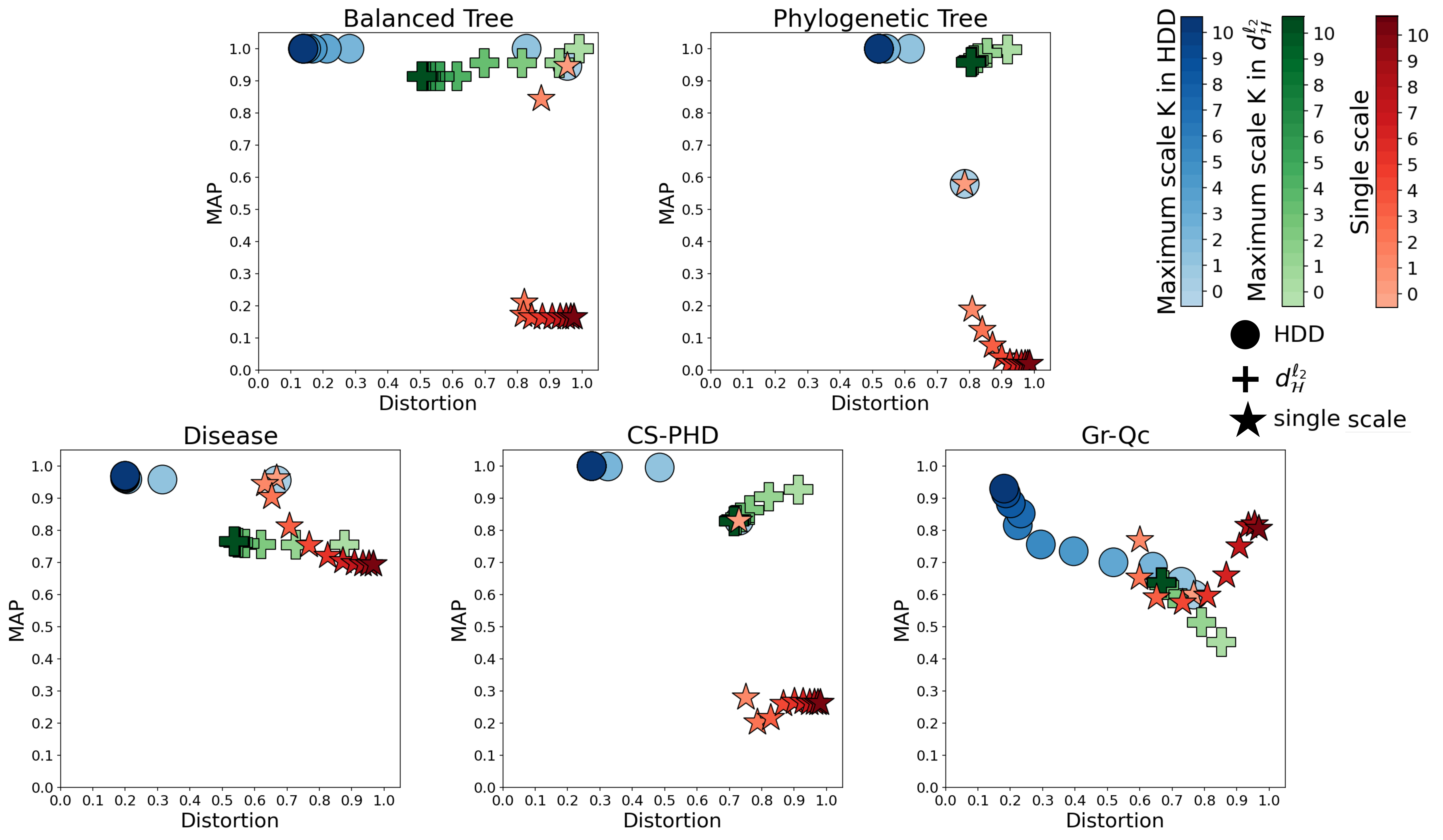}
        \caption{
         Distortion-MAP plots of the five datasets for graph embedding. In each plot, the blue circle,
green plus, and red star represent the result of HDD, $d_\mathcal{H}^{\ell_2}$, and the single scale embedding, respectively. The color of the points
represents the parameter $K$ (resp. $k$) for HDD and $d_\mathcal{H}^{\ell_2}$ (resp. single scale embedding).}
        \label{fig:ablation_map_distortion_graph}
\end{figure}

In addition, we conducted experiments that compare the performance of the proposed HDD with a variant in which the hyperbolic distance is replaced by the following Euclidean distance
\begin{equation}
    d_{\text{Euc}}(i,i') = \left\lVert \zeta_K(x_i)- \zeta_K(x_{i'}) \right\rVert_2. 
\end{equation}
    Our results are presented in Table~\ref{tab:hierarchical_graph_MAP_distortion_compare_to_Euc}, where we can see that using the Euclidean distance does not capture the hierarchical structure. 
    This empirical evidence demonstrates the importance of the hyperbolic distance in our method, showing that the proposed construction of HDD is essential to the recovery of the hierarchy.

\begin{table}[H]
% \centering
\captionof{table}{MAP and average distortion ($D_{\text{avg}}$) of hierarchical graph embedding using the Euclidean distance. } 
\begin{center}
\resizebox{0.6\textwidth}{!}{%
\begin{tabular}{lcccccccccr}
\toprule
&  \multicolumn{1}{c}{\quad}   & \multicolumn{1}{c}{Balanced tree} & \multicolumn{1}{c}{Phylo tree } & \multicolumn{1}{c}{Diseases} &   \multicolumn{1}{c}{CS-PhD } &  \multicolumn{1}{c}{Gr-QC }  \\
\midrule
\multirow{2}{*}{\STAB{\rotatebox[origin=c]{90}{MAP}}} & HDD & 1.0 & 1.0  & 0.970 & 0.999 & 0.930 \\
& Euc & 0.219 & 0.154 & 0.132  & 0.116 & 0.228\\
\midrule
\multirow{2}{*}{\STAB{\rotatebox[origin=c]{90}{$D_{\text{avg}}$}}}  & HDD & 0.144 & 0.520 & 0.206  & 0.274 &0.179 \\
& Euc & 0.694 & 0.712 & 0.736 & 0.688 & 0.781\\
\bottomrule
\end{tabular}}
\end{center}
\label{tab:hierarchical_graph_MAP_distortion_compare_to_Euc}
\end{table}

\subsection{Single-Cell Gene Expression Data}
The obtained MAP, average distortion, and classification accuracy of the scRNA-seq datasets are reported in Table~\ref{tab:sc_rna_performance}.
% In addition, the run time of HDD and the competing baselines is presented in Table~\ref{tab:sc_rna_runtime}.
In the gene expression data, the maximum scales used in Algorithm \ref{alg:HDD} for the Zeisel and the CBMC datasets are set to $K=9$ and $K=13$, respectively. 
We note that they are slightly larger than the scales used in the graph embedding experiment due to the larger size and dimensionality of the data. 
Observing the table, we see that HDD achieves the best MAP and the second-best average distortion. 
In terms of classification accuracy, HDD outperforms all the competing methods in both scRNA-seq datasets. 
We report the run time of HDD and the competing baselines in Table~\ref{tab:sc_rna_runtime}.
Note that the optimization-based methods, PE, PT, and HHC, require a much longer time to find the hierarchical representation, similar to the hierarchical graph embedding task in Table~\ref{tab:hierarchical_graph_runtime}. 
Yet, the additional computational time does not lead to improved embedding quality and downstream classification accuracy. 
Our HDD obtains a slightly larger distortion than TR, and it is slower than TR. Yet, its advantage in terms of MAP and classification accuracy is significant, as illustrated in Fig.~\ref{fig:graph_gene}, Fig.~\ref{fig:graph_gene_classification_acc}, and Table~\ref{tab:sc_rna_performance}. 

\begin{table}[H]
% \centering
\caption{MAP, average distortion ($D_{\text{avg}}$), and classification accuracy of scRNA-seq datasets. } 
\begin{center}
\resizebox{0.8\textwidth}{!}{%
\begin{tabular}{lcccccccccr}
\toprule
&  \multicolumn{1}{c}{\quad}   & \multicolumn{1}{c}{HDD} & \multicolumn{1}{c}{TR} & \multicolumn{1}{c}{PE-2} & \multicolumn{1}{c}{PE} &\multicolumn{1}{c}{hMDS}  &   \multicolumn{1}{c}{PT} &  \multicolumn{1}{c}{HHC} \\
\midrule
\multirow{3}{*}{\STAB{\rotatebox[origin=c]{90}{Zeisel}}} & MAP & \textbf{0.996} & 0.803 & 0.779 &  0.788  &   0.710  & 0.542 & 0.853\\
& $D_{\text{avg}}$ & 0.169 & \textbf{0.121} &  0.278 & 0.223  & 0.213 & 0.581 & 0.482\\
% & Time (s) & 1.13 $\cdot 10^2$ & \textbf{1.09 $\cdot 10^0$} & 2.13 $\cdot 10^4$ &  2.21 $\cdot 10^4$  &  1.07 $\cdot 10^2$ & 1.07 $\cdot 10^4$ & 1.69 $\cdot 10^4$ \\
& ACC. & \textbf{0.862\scriptsize{$\pm$0.014}}   & 0.664\scriptsize{$\pm$0.039} & 0.712\scriptsize{$\pm$0.018} &  0.743\scriptsize{$\pm$0.018}  &  0.802\scriptsize{$\pm$0.041} & 0.597\scriptsize{$\pm$0.098} & 0.811\scriptsize{$\pm$0.039}\\
\midrule
\multirow{3}{*}{\STAB{\rotatebox[origin=c]{90}{CBMC}}} & MAP &\textbf{ 0.979} & 0.713 & 0.749 &  0.817  &  0.789 & 0.760 & 0.806\\
& $D_{\text{avg}}$ & 0.297 & \textbf{0.254} &  0.522 & 0.489  &  0.364 & 0.473 & 0.323\\
% & Time (s) & 5.91 $\cdot 10^2$  & \textbf{2.48 $\cdot 10^0$} & 4.26 $\cdot 10^4$ &  4.66 $\cdot 10^4$  &  4.13 $\cdot 10^2$ & 3.73 $\cdot 10^4$ &  4.54 $\cdot 10^4$ \\
& ACC. & \textbf{0.832\scriptsize{$\pm$0.023}}   & 0.741\scriptsize{$\pm$0.037} & 0.739\scriptsize{$\pm$0.061} &  0.752\scriptsize{$\pm$0.019}  &  0.733\scriptsize{$\pm$0.039} & 0.648\scriptsize{$\pm$0.027} & 0.788\scriptsize{$\pm$0.029}\\
\bottomrule
\end{tabular}}
\end{center}
\label{tab:sc_rna_performance}
\end{table}

\begin{table}[H]
% \centering
\caption{Execution time (in seconds) of scRNA-seq datasets.} 
\begin{center}
\resizebox{0.8\textwidth}{!}{%
\begin{tabular}{lllccccccccccccr}
\toprule
&  \multicolumn{2}{c}{Dataset \scriptsize{($\#$Points, $\#$Classes)}} & HDD & TR  & PE-2 & PE &  hMDS   & PT & HHC\\
\cmidrule{4-10}
& Zeisel &(3005, 7)& 1.13 $\cdot 10^2$ & \textbf{1.09 $\cdot 10^0$} & 2.13 $\cdot 10^4$ &  2.21 $\cdot 10^4$  &  1.07 $\cdot 10^2$ & 1.07 $\cdot 10^4$ & 1.69 $\cdot 10^3$ \\
& CBMC &(8617, 13) & 5.91 $\cdot 10^2$  & \textbf{2.48 $\cdot 10^0$} & 4.26 $\cdot 10^4$ &  4.66 $\cdot 10^4$  &  4.13 $\cdot 10^2$ & 3.73 $\cdot 10^4$ &  4.54 $\cdot 10^3$ \\
\bottomrule
\end{tabular}}
\end{center}
\label{tab:sc_rna_runtime}
\end{table}

\end{document}